\documentclass{article}

\PassOptionsToPackage{numbers,sort&compress}{natbib}
\usepackage[preprint]{neurips_2026}

\usepackage[utf8]{inputenc} 
\usepackage[T1]{fontenc}    
\usepackage{hyperref}       
\usepackage{url}            
\usepackage{booktabs}       
\usepackage{amsfonts}       
\usepackage{nicefrac}       
\usepackage{microtype}      
\usepackage{xcolor}         
\usepackage{enumitem}
\usepackage{amsmath}
\usepackage{amssymb}
\usepackage{mathtools}
\usepackage{amsthm}
\allowdisplaybreaks
\usepackage{wrapfig}
\usepackage{graphicx}
\usepackage{algorithm}
\usepackage{algorithmic}
\usepackage{subfigure}
\usepackage{wrapfig}
\theoremstyle{plain}
\newtheorem{theorem}{Theorem}[section]

\newtheorem{lemma}[theorem]{Lemma}
\newtheorem{corollary}[theorem]{Corollary}
\theoremstyle{definition}

\newtheorem{assumption}[theorem]{Assumption}
\theoremstyle{remark}

\title{Interactive Inverse Reinforcement Learning of Interaction Scenarios via Bi-level Optimization}

\author{%
Yue Mao$^{1}$ \quad Shicheng Liu$^{2}$ \quad Siyuan Xu$^{2}$ \quad Minghui Zhu$^{2}$\\
  $^{1}$Department of Mechanical Engineering, Pennsylvania State University \\
  $^{2}$Department of Electrical Engineering, Pennsylvania State University \\
  \texttt{\{ypm5140, sfl5539, spx5032, muz16\}@psu.edu}
}

\begin{document}

\maketitle

\begin{abstract}
Inverse reinforcement learning (IRL) learns a reward function and a corresponding policy that best fit the demonstration data of an expert. However, in the current IRL setting, the learner is isolated from the expert and can only passively observe the expert demonstrations. This limits the applicability of IRL to interactive settings, where the learner actively interacts with the expert and needs to infer the expert’s reward function from the interactions. To bridge the gap, this paper studies interactive IRL (IIRL) where a learner aims to learn the reward function of an expert and a policy to interact with the expert during its interactions with the expert. We formulate IIRL as a stochastic bi-level optimization problem where the lower level learns a reward function to explain the behaviors of the expert, and the upper level learns a policy to interact with the expert. We develop a double-loop algorithm, Bi-level Interactive Scenarios Inverse Reinforcement Learning (BISIRL), which solves the lower-level problem in the inner loop and the upper-level problem in the outer loop. We formally guarantee that BISIRL converges at a rate of $\mathcal{O}(1/\sqrt{K})$ and validate our algorithm through extensive experiments. 
\end{abstract}

\section{Introduction}\label{intro}
Inverse reinforcement learning (IRL) is a framework in which a learner seeks to recover a reward function and a corresponding policy that are consistent with the demonstrated trajectories of an expert. IRL has been successfully applied across a wide range of domains, including robotics \cite{ziebart2008maximum, okal2016learning}, cybersecurity \cite{zhang2019non, elnaggar2018irl}, and biology \cite{hirakawa2018can, ashwood2022dynamic}.

In these classical IRL settings, the learner is assumed to be a passive observer of expert demonstrations and cannot influence the expert’s behavior. However, many emerging applications motivate a more interactive paradigm, in which the learner actively interacts with the expert and learns the expert’s reward function during the interactions. For example, consider a human–robot collaboration scenario \cite{buning2022interactive}, where a robot assists a human in navigating a maze by opening doors that lead to different paths. The human cannot open doors independently and must rely on the robot to select the correct doors to reach the destination. In this setting, the robot and human operate within the same environment, and the robot’s actions directly influence the human’s behavior. To assist effectively, the robot needs to infer the human’s destination from observed interactions and adapt its actions accordingly. A similar challenge arises in cybersecurity applications \cite{zhang2019non}, where a defender (learner) learns the attack goal of an attacker (expert) and dynamically adjusts its defense strategy while defending against the attacker. In both examples, the learner’s behavior affects the expert’s actions, and learning must occur online through interactions. However, existing IRL frameworks are not designed to handle such interactive settings, as they restrict the learner to passive observation and prohibit learning during interaction.

To bridge the gap, this paper studies interactive IRL (IIRL), where a learner aims to learn a reward function to explain the expert’s behavior and a policy to effectively interact with the expert during its interactions with the expert. The recent works study fully cooperative IIRL \cite{palaniappan2017efficient,buning2022interactive,kamalaruban2019interactive,hadfield2016cooperative} where the learner and expert share an identical reward function, and fully competitive IIRL \cite{zhang2019non,wang2018competitive} where the learner’s reward is the negative of the expert’s reward. However, many real-world interactions fall between these two extremes. In such settings, the learner and the expert have distinct and potentially misaligned objectives that are neither fully cooperative nor fully competitive. For example, in a traffic interaction scenario \cite{el2020towards}, a vehicle (learner) and a pedestrian (expert) navigate an intersection toward different destinations. The vehicle must yield to the pedestrian to ensure safety, while still pursuing its own navigation goal. This interaction involves partially aligned but non-identical reward functions. Motivated by such scenarios, we study the general case in which the reward functions of the learner and the expert are arbitrary. We summarize our contributions as follows.

\textbf{Contribution statement}. Our contributions are threefold.
 
First, we study an IIRL problem where the learner learns a reward function to explain the expert's behaviors and a policy to interact with the expert. We formulate this setting as a stochastic bi-level optimization problem, where the lower level learns the expert’s reward function and the upper level optimizes the learner’s interaction policy.
 
Second, we propose a double-loop algorithm, termed Bi-level Interactive Scenarios Inverse Reinforcement Learning (BISIRL), to solve the bi-level optimization problem. A key challenge lies in computing the hypergradient of the upper-level objective, which involves an intractable Hessian. To address this challenge, we propose a hypergradient approximation based on simultaneous perturbation stochastic approximation (SPSA) \cite{spall2000adaptive}. We theoretically prove that this approximation reduces the computational complexity of hypergradient computation from $\mathcal{O}(H^2)$ to $\mathcal{O}(H)$, where $H$ is the trajectory length. 

Third, we theoretically guarantee that BISIRL converges at a rate of $\mathcal{O}(1/\sqrt{K})$, where $K$ is the iteration number. We further validate the proposed approach through four experiments, demonstrating that BISIRL achieves performance comparable to baselines that require additional information.

\section{Related work}
\textbf{IRL and multiagent IRL (MA-IRL).} As \cite{ng2000algorithms} mentioned, IRL faces the challenge that the demonstrated trajectories can be explained by multiple reward functions. Several approaches have been proposed to address this challenge.  The current state-of-the-art IRL methods include maximum margin IRL \cite{ng2000algorithms,abbeel2004apprenticeship}, maximum entropy IRL \cite{ziebart2008maximum,ziebart2010modeling}, maximum likelihood IRL \cite{zeng2022maximum,zeng2023demonstrations}, and Bayesian IRL \cite{ramachandran2007bayesian,choi2012nonparametric}. However, these methods all focus on single-expert scenarios. MA-IRL extends IRL to settings with multiple experts, where one or more learners recover the reward functions of multiple experts from their demonstrations \cite{lin2019multi,yu2019multi,liu2022distributed}. These MA-IRL approaches still assume that the learners are isolated from the experts (i.e., no direct interaction between them) and thus cannot address interactive IRL scenarios.

\textbf{Fully cooperative and fully competitive IIRL.} Prior work on IIRL has primarily explored two cases: fully cooperative and fully competitive scenarios. In the fully cooperative setting, the learner and expert share the same reward function, and their interaction is formulated as a cooperative game \cite{palaniappan2017efficient,buning2022interactive,kamalaruban2019interactive,hadfield2016cooperative}. In this case, the learner’s objective is to recover the expert’s reward function. In contrast, the fully competitive setting assumes the learner’s reward is the negative of the expert’s reward, leading to a zero-sum game \cite{zhang2019non,wang2018competitive}. Here, the learner aims to learn an objective that directly opposes the expert’s reward. Both of these approaches impose strong assumptions on the relationship between the expert’s and learner’s reward functions, specifically assuming they are either identical or exact opposites. In this paper, the relationship between the reward functions of the learner and the expert is arbitrary.

\textbf{Bi-level optimization.} Bi-level optimization has been applied to many machine learning problems, including meta-learning \cite{lee2019meta,xu2022meta}, hyperparameter optimization \cite{pedregosa2016hyperparameter, Xu_Zhu_2023}, and IRL \cite{liu2022distributed,liu2024trajectory}.  A classic approach to such problems is the descent method \cite{kolstad1990derivative,Xu_Zhu_2023}. This method typically requires computing the second-order Hessian of the lower-level objective function, an operation that is prohibitively expensive to compute in our setting. A common strategy to avoid explicitly computing the Hessian is to approximate it using finite differences \cite{strikwerda2004finite}. In this paper, we further reduce the computational burden by adopting the SPSA method, which perturbs all dimensions of the decision variable simultaneously and therefore requires only two objective function evaluations per iteration.

\section{Model and problem statement}\label{mps}

In this section, we introduce the IIRL problem.

\textbf{Markov Game}. We model the interactions between a learner and an expert as a finite horizon Markov Game (MG) $(S, A, P, H, r_l,r_e, \gamma)$. The elements of the MG are defined as follows:
\begin{itemize}[noitemsep]
    \vspace{-3mm}
    \item $S\triangleq S_l \times S_e$ is the state space where $S_l$ and $S_e$ are the state spaces of the learner and the expert, respectively. We denote $s=(s_l,s_e) \in S,$ with $s_l\in S_l$ and $s_e\in S_e$.
    \item $A\triangleq A_l \times A_e$ is the joint action space, where $A_l$ and $A_e$ are the action spaces of the learner and the expert, respectively. We denote $a=(a_l,a_e) \in A,$ with $a_l\in A_l$ and $a_e\in A_e$.
    \item $P(s'|s,a)$ is the transition probability density for moving from state $s$ to $s'$ by taking a joint action $a$.
    \item $H$ is the finite time horizon.
    \item $r_l$ is the learner's reward function, mapping state-action pairs $(s, a)$ to bounded rewards.
    \item $r_e$ is the expert's reward function, mapping state-action pairs $(s, a)$ to bounded rewards.
    \item $\gamma\in (0,1]$ is the discount factor.
    \vspace{-3mm}
\end{itemize}

We denote $\pi_l(a_l|s)$ as the learner's policy and $\pi_e(a_e|s)$ as the expert's policy. The joint policy is defined as $\pi(a|s) \triangleq \pi_l(a_l|s) \times \pi_e(a_e|s)$, which represents the probability of the learner taking action $a_l$ and the expert taking $a_e$ at state $s$. When the joint policy $\pi$ is executed, the MG generates a trajectory $\zeta = s^{0}, a^{0}, s^{1}, a^{1}, \cdots, s^{(H-1)}, a^{(H-1)}$.

The learner aims to maximize the cumulative reward $E^{\pi_l,\pi_e} [\sum_{h=0}^{H-1} \gamma^hr_l(s^h,a^h)]$ to obtain its optimal policy $\pi_l^*$. Analogously, the expert aims to maximize $E^{\pi_l,\pi_e} [\sum_{h=0}^{H-1} \gamma^h r_e(s^h,a^h)]$ to obtain its optimal policy $\pi_e^*$. We define the interacting policy $\pi^*\triangleq \pi_l^*\times\pi_e^*$. 

\textbf{Knowledge and goal of the learner}. The learner does not have access to the expert’s reward function $r_e$ but can access its own reward $r_l(s,a)$ for any state-action pair $(s,a)$. It interacts with the expert and can observe the corresponding trajectories. Based on the trajectories, the learner aims to recover the expert's reward function $r_e$ and compute the interacting policy $\pi^*$. Note that the learner not only wants to learn its own part $\pi_l^*$ of the interacting policy but also wants to imitate the expert's part $\pi_e^*$ of the interacting policy (as imitating the expert is a standard goal of IRL).

\section{Problem Formulation}
In this section, we formulate the learning problem in Section \ref{mps} as a bi-level optimization problem where the lower-level optimization problem learns the expert reward function $r_e$ and the upper-level optimization problem learns the interacting policy $\pi^*$.

To this end, the learner uses parameterized reward models $r_{\theta_e}$ and $r_{\theta_l}$ to estimate expert's reward function $r_e$ and learner's reward function $r_l$. The reward parameters satisfy $\theta_e,\theta_l \in \Theta\triangleq\{\theta | \|\theta\|_2 \leq 1\}$. Given the reward functions $(r_{\theta_l},r_{\theta_e})$, let $\pi_{\theta_l,\theta_e}$ denote the policy induced by the corresponding MG. This policy $\pi_{\theta_l,\theta_e}$ is used to estimate $\pi^*$.

  \begin{wrapfigure}[15]{r}{0.5\textwidth}
 \vspace{-4mm}
\centering
\includegraphics[width=1\linewidth]{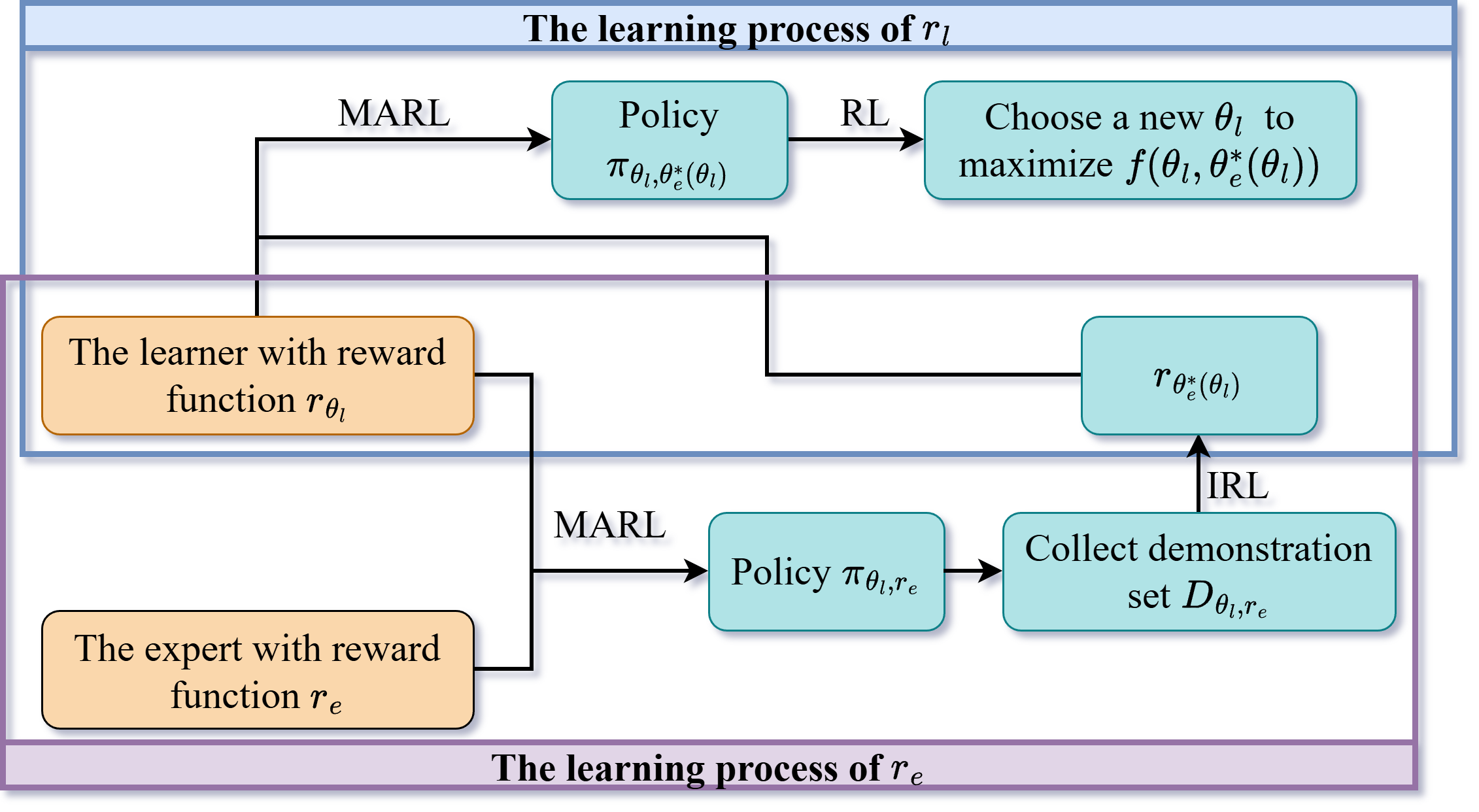}
\vspace{-6mm}
\caption{Flowchart of the overall learning process. The process of learning $r_{l}$ is included in the upper block, and the process of learning $r_e$ is included in the lower block.}
\label{fig:fc}
\end{wrapfigure}

\textbf{Remark on learning the learner's reward function $r_{\theta_l}$}. At first glance, learning $r_{\theta_l}$ may appear unnecessary, since the learner can directly observe its reward values $r_l(s,a)$ and could in principle learn its policy via standard RL. However, our objective is to learn the interacting policy $\pi^*$, which is defined as a joint policy induced by an MG with reward functions $(r_l,r_e)$. To preserve this structure, we estimate $\pi^*$ using $\pi_{\theta_l,\theta_e}$, which is induced by an MG parameterized by $(\theta_l,\theta_e)$. If the learner were to directly learn its policy $\pi_l$ without learning $r_{\theta_l}$, there would be no guarantee that the resulting joint policy, composed of $\pi_l$ and the expert policy, corresponds to any MG. Learning $r_{\theta_l}$ therefore ensures structural consistency between the learned policy and the underlying interacting policy $\pi^*$.

In the following context, we first elaborate the lower-level optimization problem that learns $r_{\theta_e}$ to estimate the expert's reward function, and then the upper-level optimization problem that learns $r_{\theta_l}$ to estimate the learner's reward function and the joint policy $\pi_{\theta_l,\theta_e}$ to estimate $\pi^*$.

\textbf{Lower-level optimization: learning the expert's reward function $r_e$}. Given a learner's reward estimate $r_{\theta_l}$ from the upper level, the lower level optimizes the parameterized reward model $r_{\theta_e}$ to learn the expert's reward function. Specifically, for a given pair of parameters $(\theta_l,\theta_e)$, the learner can compute the joint policy $\pi_{\theta_l,\theta_e}$. The learner then executes its own part of the joint policy $\pi_{\theta_l,\theta_e}$ to interact with the expert (who acts according to the ground-truth reward function $r_e$) and collect a set of interaction trajectories $D_{\theta_l, r_{e}} \triangleq \{\zeta_i\}_{i=1}^{d}$. Since the expert follows the ground-truth reward function $r_e$, the demonstration set $D_{\theta_l, r_{e}}$, serve as demonstrations for learning $r_e$. Given that the learner knows $r_{\theta_l}$, inferring $r_e$ from collected demonstrations is a special case of the standard two-agent IRL problem \cite{lin2019multi,yu2019multi,liu2022distributed} with one of the reward functions known to the learner. This learning procedure, depicted in the lower block of Figure \ref{fig:fc}, can be formulated as the following maximum likelihood IRL (ML-IRL) problem:
 \begin{equation}\label{mlirl}
     \theta^*_e(\theta_l) = \underset{\theta_e \in \Theta}{\arg\min} \quad L(\theta_l,\theta_e),
 \end{equation}
where the loss function of the lower level is $L(\theta_l,\theta_e) \triangleq -\sum_{i=1}^d\sum_{h=0}^{H-1} [\ln  \pi_{\theta_l,\theta_e} (a^{ih}|s^{ih})]+ \frac{\lambda}{2}\|\theta_e\|_2^2$ and the state-action pairs $ (a^{ih},s^{ih}) \in \zeta_i \in D_{\theta_l, r_{e}}$. The problem (\ref{mlirl}) learns a reward function $r_{\theta^*_e(\theta_l)}$ to maximize the likelihood of the interaction trajectories $D_{\theta_l, r_{e}}$. With a given $r_{\theta_l}$, the policy $\pi_{\theta_l, \theta_e^*( \theta_l)}$ induced by the learned expert's reward function $r_{\theta_e^*(\theta_l)}$ is expected to generate same trajectories as the expert. Notice that the maximum likelihood estimation has been extensively applied in IRL \cite{ziebart2008maximum, ziebart2010modeling}. Additionally, to promote simpler reward structures, an $L_2$ regularization term $ \frac {\lambda}{2}\|\theta_e\|_2^2$ with a small positive scalar $\lambda$ is incorporated.  

\textbf{Upper-level optimization: learning the learner's reward function $r_l$ and the joint policy $\pi^*$}. Given a learner's reward parameter $\theta_l$, the learner solves problem (\ref{mlirl}) to obtain the expert's reward parameter $\theta_e^*(\theta_l)$ and the corresponding policy $\pi_{\theta_l, \theta_e^*( \theta_l)}$. The learner's objective is to learn a parameter $\theta_l$ such that the policy $\pi_{\theta_l, \theta_e^*( \theta_l)}$ maximizes the cumulative reward $E^{\pi_{\theta_l, \theta_e^*( \theta_l)}}[\sum_{h=0}^{H-1} \gamma^h r_l(s^h,a^h)]$. This learning process is depicted in the upper block of Figure \ref{fig:fc} and formulated as follows:
\begin{equation}\label{maxf}
   \underset{\theta_l \in \Theta}{\arg\min} \quad  f(\theta_l, \theta_e^*( \theta_l)) \triangleq -E^{\pi_{\theta_l, \theta_e^*( \theta_l)}}[\sum_{h=0}^{H-1} \gamma^h r_l(s^h,a^h)].
\end{equation}
To solve the problem (\ref{maxf}), the learner needs to interact with the expert and receive the reward values of $r_l(s^h,a^h)$. Note that $\theta_l$ is not the parameter of the policy $\pi_{\theta_l, \theta_e^*( \theta_l)}$ but a hyperparameter. Specifically, $\theta_l$ is the parameter of the learner's reward function $r_{\theta_l}$ and $\pi_{\theta_l, \theta_e^*( \theta_l)}$ is the joint policy induced by the MG parameterized by $(\theta_l,\theta_e^*(\theta_l))$. Therefore, the policy $\pi_{\theta_l, \theta_e^*( \theta_l)}$ is hyper-parameterized by $\theta_l$.

\textbf{Bi-level optimization}. The entire learning procedure exhibits a hierarchical structure. Given the current learner's reward parameter $\theta_l$, the lower level solves the ML-IRL problem (\ref{mlirl}) to determine the expert's reward function $r_{\theta_e^*(\theta_l)}$ and the joint policy $\pi_{\theta_l,\theta_e^*(\theta_l)}$. The upper level optimizes the parameter $\theta_l$ such that $\pi_{\theta_l,\theta_e^*(\theta_l)}$ can maximize the cumulative reward in the problem (\ref{maxf}). This hierarchical optimization framework is illustrated in Figure \ref{fig:fc} and can be formulated as a bi-level optimization problem:

  \begin{equation}\label{eq:3}
          \underset{\theta_l \in \Theta}{\arg\min} \quad  f(\theta_l, \theta_e^*( \theta_l)),\quad
          \textrm{s.t.}  \quad \theta^*_e(\theta_l) = \underset{\theta_e \in \Theta}{\arg\min} \quad L(\theta_l,\theta_e).
\end{equation}

\section{Algorithm}
In this section, we develop a double-loop algorithm named BISIRL (Algorithm \ref{alg:1}) to solve the bi-level optimization problem described in the problem (\ref{eq:3}). In each $k$-th outer loop iteration, the learner partially solves the lower-level optimization problem within an inner loop and subsequently employs this intermediate solution to tackle the upper-level optimization problem through an outer loop. The specifics of the inner and outer loops are detailed in Sections \ref{alginner} and \ref{algouter}, respectively.

\subsection{Inner loop}\label{alginner}
At each iteration $t$ of the inner loop, the learner updates the parameter $\theta_e(t)$ via projected SGD. Specifically, $\theta_e(t)$ is updated by moving in the opposite direction of the partial gradient  $\nabla_{\theta_e}L(\theta_l(k),\theta_e(t))$ and projecting the resulting parameter back onto the feasible set $\Theta$. The analytical expression of the gradient $\nabla_{\theta_e}L(\theta_l,\theta_e)$ is provided in the following lemma. 

\begin{lemma}\label{lemma1}
    The gradient $\nabla_{\theta_e}L(\theta_l,\theta_e)=\mu_e(\pi_{\theta_l, \theta_e})-\hat{\mu}_e(D_{\theta_l, r_{e}})+ \lambda \theta_e$, where $\mu_e(\pi_{\theta_l, \theta_e}) \triangleq E^{\pi_{\theta_l, \theta_e}}[\sum^{H-1}_{h=0} \gamma^h \nabla_{\theta_e} r_{\theta_e}(s^{h},a^{h})]$, $\hat{\mu}_e(D_{\theta_l, r_{e}}) \triangleq \frac{1}{d}\sum_{i=1}^d\sum_{h=0}^{H-1} \gamma^h \nabla_{\theta_e} r_{\theta_e}(s^{ih}, a^{ih})$, and $(s^{ih}, a^{ih})\in \zeta_i \in D_{\theta_l, r_{e}} $
\end{lemma}
The inner loop uses $t_k$-step projected gradient descent $\theta_{e}(t+1) = \Pi_{\Theta}(\theta_{e}(t) - \beta_t \nabla_{\theta_e}L(\theta_{l}(k), \theta_{e}(t)))$ to obtain an expert reward estimate $r_{\theta_e(t_k-1)}$.

\begin{algorithm}[tb]
\caption{ Bi-level Interactive Scenarios Inverse Reinforcement Learning (BISIRL)}\label{alg:1}
\begin{algorithmic}
\renewcommand{\baselinestretch}{1.15}\selectfont
\STATE {\bfseries Initializes} $\theta_{l}(0) \in \Theta,\theta_e(0)\in \Theta$, step size sequence $\{\alpha_k\}, \{\beta_t\}$, regularization parameter $\lambda$ and integer sequence $\{t_k\}$
\FOR{$k=0,1,\cdots, K-1$}
    \STATE $\theta_e(0) = \theta_{e}(k)$
    \STATE Samples the trajectory set $D_{\theta_l(k), r_{e}}$
    \FOR{$t=0,\cdots,t_k-1$}
        \STATE Calculates $\nabla_{\theta_e}L(\theta_{l}(k), \theta_{e}(t))$ with Lemma \ref{lemma1}
        \STATE $\theta_{e}(t+1) = \Pi_{\Theta}(\theta_{e}(t) - \beta_t \nabla_{\theta_e}L(\theta_{l}(k), \theta_{e}(t)))$
    \ENDFOR
    \STATE $\theta_{e}(k) = \theta_{e}(t_k-1)$
    \STATE {\bfseries Initializes} the random vector $\Delta(k)$ and the positive scalar $p(k)$
    \STATE {\bfseries Gets} $\hat{\nabla}^2_{\theta_l\theta_e}L(\theta_l(k), \theta_e(k))$, $\hat{\nabla}^2_{\theta_e}L(\theta_l(k), \theta_e(k))$, $\hat{\nabla}_{\theta_l} f(\theta_l(k),\theta_e(k))$  and $\hat{\nabla}_{\theta_e(k)} f(\theta_l(k),\theta_e(k))$ through SPSA approximation following equation (\ref{eq:4})
    \STATE {\bfseries Calculates} $[\hat{\nabla}^2_{\theta_e}L(\theta_l(k), \theta_e(k))]^{-1}\hat{\nabla}_{\theta_e} f(\theta_l(k),\theta_e(k))$ following equation (\ref{conju}) 
    \STATE {\bfseries Calculates} $\hat{\nabla}f(\theta_{l}(k),\theta_{e}(k))$ following equation (\ref{ehyper}) 
    \STATE $\theta_{l}(k+1) = \Pi_{\Theta} (\theta_{l}(k) - \alpha_k \hat{\nabla} f(\theta_{l}(k),\theta_{e}(k)) )$
\ENDFOR
\end{algorithmic}
\end{algorithm}

\subsection{Outer loop}\label{algouter}
At $k$-th iteration of the outer loop, the learner updates the parameter $\theta_l(k)$ via projected SGD. Similar to the inner loop, this update requires computing the hypergradient $\nabla f(\theta_{l}(k),\theta_{e}(k))$. The analytical form of $\nabla f(\theta_{l}(k),\theta_{e}(k))$, which is widely used in bi-level optimization problems, is presented in the following lemma. 
\begin{lemma}\label{lemma2}
    The analytical expression of the hypergradient $\nabla f(\theta_l,\theta_e)$ used for updating $\theta_{l}$ is $ \nabla_{\theta_l} f(\theta_l,\theta_e)- \nabla^2_{\theta_l\theta_e}L(\theta_l, \theta_e)[\nabla^2_{\theta_e}L(\theta_l, \theta_e)]^{-1} \nabla_{\theta_e}f(\theta_l,\theta_e)$ \cite{Xu_Zhu_2023, ghadimi2018approximation, colson2007overview}.
\end{lemma}

The computation of the hypergradient $\nabla f(\theta_{l}(k),\theta_{e}(k))$ involves the partial gradients $\nabla_{\theta_l} f(\theta_l(k),\theta_e(k)), \nabla_{\theta_e}f(\theta_l(k),\theta_e(k))$, the Jacobian $\nabla^2_{\theta_l\theta_e}L(\theta_l(k), \theta_e(k))$  and the inverse Hessian $[\nabla^2_{\theta_e}L(\theta_l(k), \theta_e(k))]^{-1}$. However, directly calculating these quantities has a large computational complexity of $\mathcal{O}(H^2)$. To address this issue, we adopt the SPSA method, which reduces computational complexity to $\mathcal{O}(H)$. A detailed discussion on the computational complexity reduction is provided later in Theorem \ref{cc}. Using SPSA, the learner obtains the estimated hypergradient $\hat{\nabla}f(\theta_{l}(k),\theta_{e}(k))$, which serves as the gradient approximation in the projected SGD update. To calculate $\hat{\nabla}f(\theta_{l}(k),\theta_{e}(k))$, the learner estimates $\hat{\nabla}_{\theta_l} f(\theta_l(k),\theta_e(k))$, $\hat{\nabla}^2_{\theta_l\theta_e}L(\theta_l(k), \theta_e(k))$, $\hat{\nabla}^2_{\theta_e}L(\theta_l(k), \theta_e(k))$, and $\hat{\nabla}_{\theta_e}f(\theta_l(k),\theta_e(k))$.
Let us take $\hat{\nabla}^2_{\theta_e}L(\theta_l(k), \theta_e(k))$ as an example to illustrate SPSA. According to the equation (2.2) in \cite{spall1992multivariate}, it is approximated as follows:
\begin{equation}\label{eq:4}
    \hat{\nabla} _ {\theta_e}^2 L(\theta_l(k), \theta_e(k)) 
    = \begin{bmatrix}
        \frac{\Delta_{d_k}}{2p\Delta_1(k)}&
        \cdots&
        \frac{\Delta_{d_k}}{2p\Delta_m (k)}
    \end{bmatrix}^T,
\end{equation}
where $\Delta_{d_k} = \nabla _ {\theta_e} L(\theta_l(k), \theta_e(k)+p(k)\Delta(k)) - \nabla _ {\theta_e} L(\theta_l(k), \theta_e(k)-p(k)\Delta(k))$,
the perturbation $\Delta(k) \in \mathbb{R}^m$ is a vector of $m$ mutually independent zero-mean random variables and each element of $\Delta(k)$ satisfies $|\Delta_{i}(k)| \leq \alpha_0$, $E|\Delta_{i}^{-1}(k)|\leq \alpha_1$, $i= 1, \cdots, m$ with $\alpha_0,\alpha_1$ as positive constants. The parameter $p(k)$ is a positive scalar. 

Referring to equation (\ref{eq:4}), the computation of the gradient $\hat{\nabla}_{\theta_l} f(\theta_l(k),\theta_e(k))$ requires the value of $f(\theta_l(k),\theta_e(k))$ with weight perturbations on $\theta_l(k)$. Analogously, computing gradients $\hat{\nabla}_{\theta_e} f(\theta_l(k),\theta_e(k))$, $\hat{\nabla}^2_{\theta_e}L(\theta_l(k), \theta_e(k))$, and $\hat{\nabla}^2_{\theta_l\theta_e}L(\theta_l(k), \theta_e(k))$ requires the values of $f(\theta_l(k),\theta_e(k))$, $\nabla_{\theta_e}L(\theta_l(k),\theta_e(k))$, and $\nabla_{\theta_l}L(\theta_l(k),\theta_e(k))$ with weight perturbations on $\theta_e(k)$, respectively. Notice that the analytical expression of $f(\theta_l(k),\theta_e(k))$ and $\nabla_{\theta_e}L(\theta_l(k),\theta_e(k))$ are shown in optimization problem (\ref{maxf}) and Lemma \ref{lemma1} respectively. Similarly to Lemma \ref{lemma1}, the gradient $\nabla_{\theta_l}L(\theta_l(k),\theta_e(k)) \triangleq \mu_l(\pi_{\theta_l(k), \theta_e(k)})-\hat{\mu}_l(D_{\theta_l(k), r_{e}})$ and the proof is given in the Appendix \ref{other g}. The expectation $\mu_l(\pi_{\theta_l, \theta_e}) \triangleq E^{\pi_{\theta_l, \theta_e}}[\sum^{H-1}_{h=0} \gamma^h \nabla_{\theta_l} r_{\theta_l}(s^{h},a^{h})]$ is the reward gradient expectation of the learner and the empirical expectation $\hat{\mu}_l(D_{\theta_l, r_{e}}) \triangleq \frac{1}{d}\sum_{i=1}^d\sum_{h=0}^{H-1} \gamma^h \nabla_{\theta_l} r_{\theta_l}(s^{ih}, a^{ih}), (s^{ih}, a^{ih})\in \zeta_i \in D_{\theta_l, r_{e}}$ is the estimated learner's reward gradient expectation with respect to the demonstration set $D_{\theta_l, r_{e}}$. 

To compute the hypergradient $\hat{\nabla} f(\theta_{l}(k),\theta_{e}(k))$, we must ensure that the estimated Hessian $\hat{\nabla}^2_{\theta_e}L(\theta_l(k), \theta_e(k))$ is invertible. Since the negative log-likelihood function is convex and the $L_2$ regularization term is strongly convex, the overall function $L(\theta_l(k), \theta_e(k))$ is strongly convex. Hence its Hessian is positive definite. By choosing suitable $\Delta(k)$ and $p(k)$, we ensure that the estimated Hessian $\hat{\nabla}^2_{\theta_e}L(\theta_l(k), \theta_e(k))$ is positive definite. Since directly inverting a matrix is computationally expensive, we use the conjugate gradient method to compute the product 
\begin{equation}\label{conju}
\begin{aligned}
        &[\hat{\nabla}^2_{\theta_e}L(\theta_l(k), \theta_e(k))]^{-1}\hat{\nabla}_{\theta_e} f(\theta_l(k),\theta_e(k))\\ &= \min_{u} \frac{1}{2}u^T\hat{\nabla}^2_{\theta_e}L(\theta_l(k), \theta_e(k))u - u^T\hat{\nabla}_{\theta_e} f(\theta_l(k),\theta_e(k)) .
\end{aligned}
\end{equation} 
Finally, using $\hat{\nabla}_{\theta_l} f(\theta_l(k),\theta_e(k))$, $\hat{\nabla}^2_{\theta_l\theta_e}L(\theta_l(k), \theta_e(k))$, and the result from conjugate gradient, we compute the estimated hypergradient 
\begin{equation}\label{ehyper}
\begin{aligned}
        \hat{\nabla}f(\theta_{l}(k),\theta_{e}(k)) = &\hat{\nabla}_{\theta_l} f(\theta_l(k),\theta_e(k)) \\
        &- \hat{\nabla}^2_{\theta_l\theta_e}L(\theta_l(k), \theta_e(k))
        [\hat{\nabla}^2_{\theta_e}L(\theta_l(k), \theta_e(k))]^{-1}\hat{\nabla}_{\theta_e}f(\theta_l(k),\theta_e(k)). 
\end{aligned}
\end{equation}

After $K$ iterations, the outer loop terminates, yielding $r_{\theta_l(K)}$, $r_{\theta_e(K)}$ and $\pi_{\theta_l(K),\theta_e(K)}$. 

\section{Analytical result}
In this section, we present our analytical results regarding computational complexity reduction and convergence rate. To facilitate our analysis, we impose the following assumption on the estimated reward functions $r_{\theta_l}$ and $r_{\theta_e}$ :
 \begin{assumption}\label{assump1}
      The estimated learner's reward function $r_{\theta_l} $ and the estimated expert's reward function $ r_{\theta_e} $ are four-times continuously differentiable, i.e., $C^4$.
 \end{assumption}
 Since $\theta_l$ and $\theta_e$ lie in compact sets, Assumption \ref{assump1} implies that the derivatives of the reward functions $r_{\theta_l}$ and $r_{\theta_e}$ are bounded. Such boundedness of higher-order derivatives is a standard assumption in the literature and has been widely adopted in bi-level optimization \cite{jin2020local,zeng2022maximum,liu2023learning}, RL \cite{wang2019neural,zhang2020global}, and IRL \cite{liu2023meta}.

\subsection{Computational complexity}
The computation complexities of computing $\nabla f(\theta_{l},\theta_{e})$ and $\hat{\nabla} f(\theta_{l},\theta_{e})$  are shown in Theorem \ref{cc}.
\begin{theorem}\label{cc}\
    Consider $H$ as the decision factor, the computational complexity of computing $\nabla f(\theta_{l},\theta_{e})$ is $\mathcal{O}(H^2)$ and that of computing $\hat{\nabla} f(\theta_{l},\theta_{e})$ is $\mathcal{O}(H)$.
\end{theorem}
Applying SPSA to approximate $\nabla f(\theta_{l},\theta_{e})$ reduces the per-iteration computational complexity of the upper-level problem from $\mathcal{O}(H^2)$ to $\mathcal{O}(H)$. The proof of Theorem \ref{cc} is in the Appendix \ref{t2}. Compared to finite-difference methods, SPSA requires fewer policies for the approximation. In SPSA, a single random perturbation of $\theta_l$ or $\theta_e$ yields two policies (one for the positive perturbation and one for the negative perturbation).  In contrast, the finite difference approach requires two policies for each parameter dimension. Each policy is generated via MARL, so the computational cost of producing these policies is non-negligible. As SPSA requires fewer policies, its overall computational cost of SPSA is lower than that of finite difference methods. Finite-difference methods, in turn, have a lower computational cost than explicitly computing gradients. 

\subsection{Convergence rate}
The convergence of our algorithm is shown in Theorem \ref{cr}.
\begin{theorem}\label{cr}
    Suppose Assumption \ref{assump1} holds, with the choices of parameters given by $p(k) = \frac{1}{k}$, $\alpha_k = \frac{1}{L_f \sqrt{K}}$, $t_k = \lceil \frac{\sqrt[4]{k+1}}{2}\rceil$, the following convergence guarantee holds:
        $\frac{1}{K}\sum_{k=0}^{K-1}E[\|\nabla f(\theta_l(k),\theta_e^*(\theta_l(k)))\|^2] \leq \mathcal{O}(\frac{1}{\sqrt{K}})$.
\end{theorem}
Theorem \ref{cr} indicates that the expected hypergradient decays at a rate of $\mathcal{O}(\frac{1}{\sqrt{K}})$, matching the convergence rate of standard bilevel optimization \cite{ghadimi2018approximation}. This implies that the bias introduced by SPSA does not slow convergence compared to standard bilevel optimization. Corollary \ref{policycon} shows that the linear reward function $r_{\theta_e}$ is a sufficient condition for the convergence of the policy $\pi_{\theta_e}$. Convergence of the cumulative reward difference has been widely used to infer convergence of the learned expert policy \cite{rhinehart2017first,renard2024convergence}.
\begin{corollary}\label{policycon}
    If the reward function $r_{\theta_e}$ is linear, the cumulative reward difference between the learned expert policy $\pi_{\theta_e}$ and $\pi_{e}$ decreases at a rate $\mathcal{O}(\frac{1}{\sqrt[4]{K}})$.
\end{corollary}

\section{Experiment}\label{experiments}
 In our experiments, we evaluate BISIRL on four widely used environments. The Multi-Agent Particle Environment (MPE) \cite{lowe2017multi,terry2021pettingzoo,yu2022surprising} involves particle agents that can move, communicate, observe one another, push each other, and interact with fixed landmarks. The StarCraft Multi-Agent Challenge (SMAC) \cite{samvelyan2019starcraft,yu2022surprising,rashid2020monotonic} is designed based on the real-time strategy game StarCraft II. The Human-Robot Interaction (HRI) environment \cite{el2020towards, fan2020distributed, zhu2021deep} models a navigation scenario in which a robot interacts with a human, and the results are listed in Appendix Section \ref{hriexp}. The security environment \cite{zhang2019non,zhang2021physical} involves automated defense systems; the results are presented in Appendix Section \ref{csexp}. The MPE, HRI, and Security environments involve mixed cooperative–competitive interactions, while the SMAC environment is fully cooperative.

We compare BISIRL with four baseline algorithms: the \textbf{MARL} algorithm \cite{lowe2017multi}, the \textbf{MA-IRL} algorithm \cite{ziebart2010modeling}, the cooperative IIRL  (\textbf{CIRL}) algorithm \cite{buning2022interactive}, and the maximum likelihood IRL (\textbf{ML-IRL}) algorithm \cite{zeng2022maximum}. Since none of these baseline algorithms can directly address our IIRL problem, we adapt the experimental setup to make them applicable. Specifically, for MARL, we assume that both the learner and the expert know their reward functions. For MA-IRL, we provide demonstrations sampled according to ground-truth reward functions of both the learner and the expert. Under this setting, the performance of MA-IRL approaches that of MARL. For CIRL, we assume that the reward function of the learner is identical to the expert's reward function. For ML-IRL, the learner has access to demonstrations generated by a policy through MARL with the learner’s initial and the ground-truth expert reward functions. 

The MARL baseline represents the best achievable performance among all considered algorithms. The MA-IRL baseline evaluates whether our proposed method can achieve comparable performance without relying on the learner's demonstrations, which are generated using the learner's ground-truth reward function. The CIRL baseline illustrates that cooperative IIRL algorithms are only effective in fully cooperative scenarios, highlighting their limitations in more general settings. Lastly, the ML-IRL baseline demonstrates that conventional IRL approaches are ill-suited for IIRL problems, as they fail to account for the mutual influence between the learner and the expert during interaction.

Since MARL serves as an upper-bound baseline and MA-IRL is expected to achieve comparable performance under our setting, the objective of BISIRL is not to outperform these two baselines, but rather to match their performance while operating under the more restrictive IIRL setting, where the privileged information used by MARL and MA-IRL is unavailable.

\begin{figure*}[htbp]

\centering
\subfigure{ 
\label{mpead}

\includegraphics[width=6.5cm]{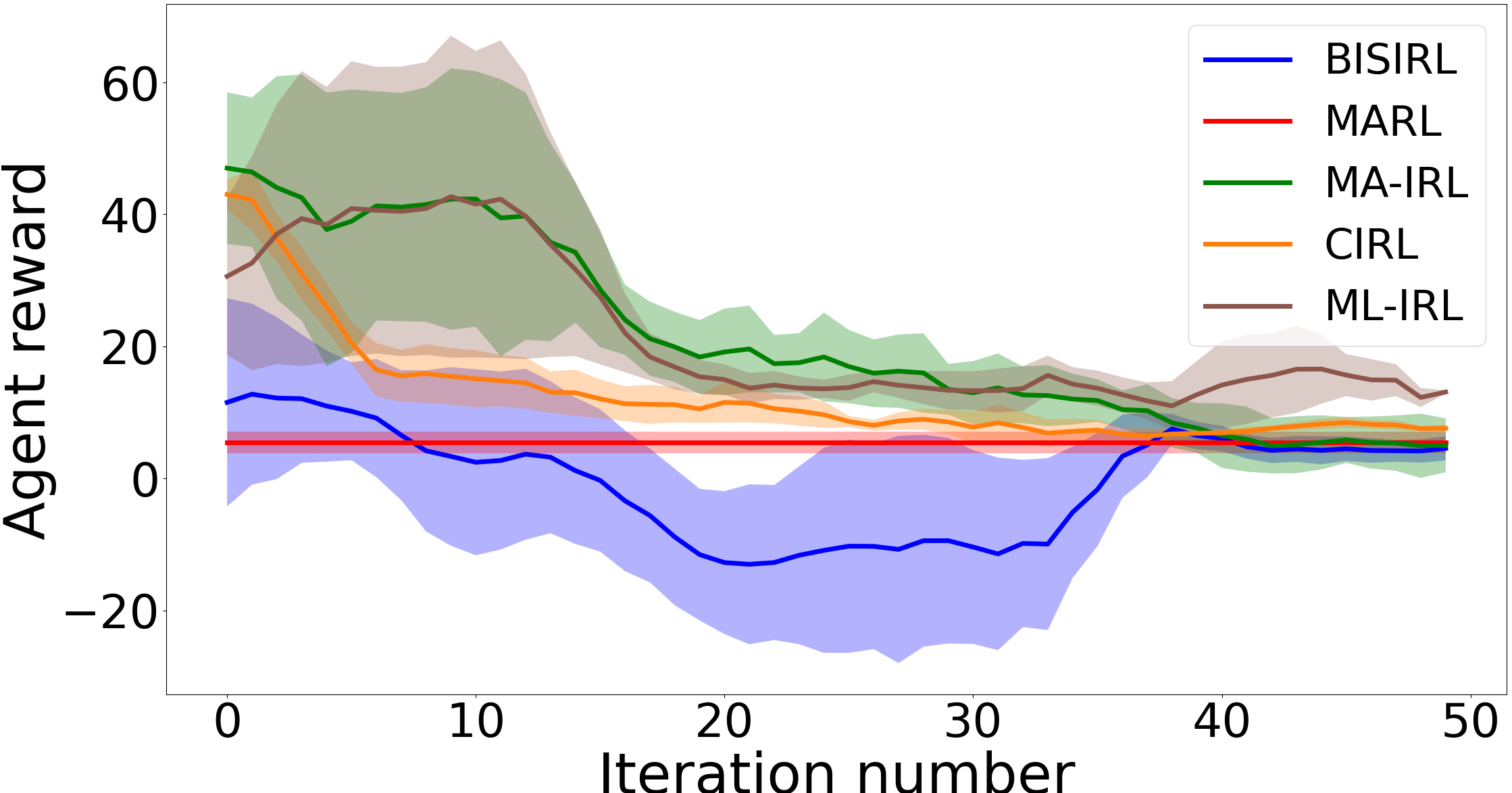}

}
\hspace{-1mm}
\subfigure{ 
\label{mpeag}
\includegraphics[width=6.5cm]{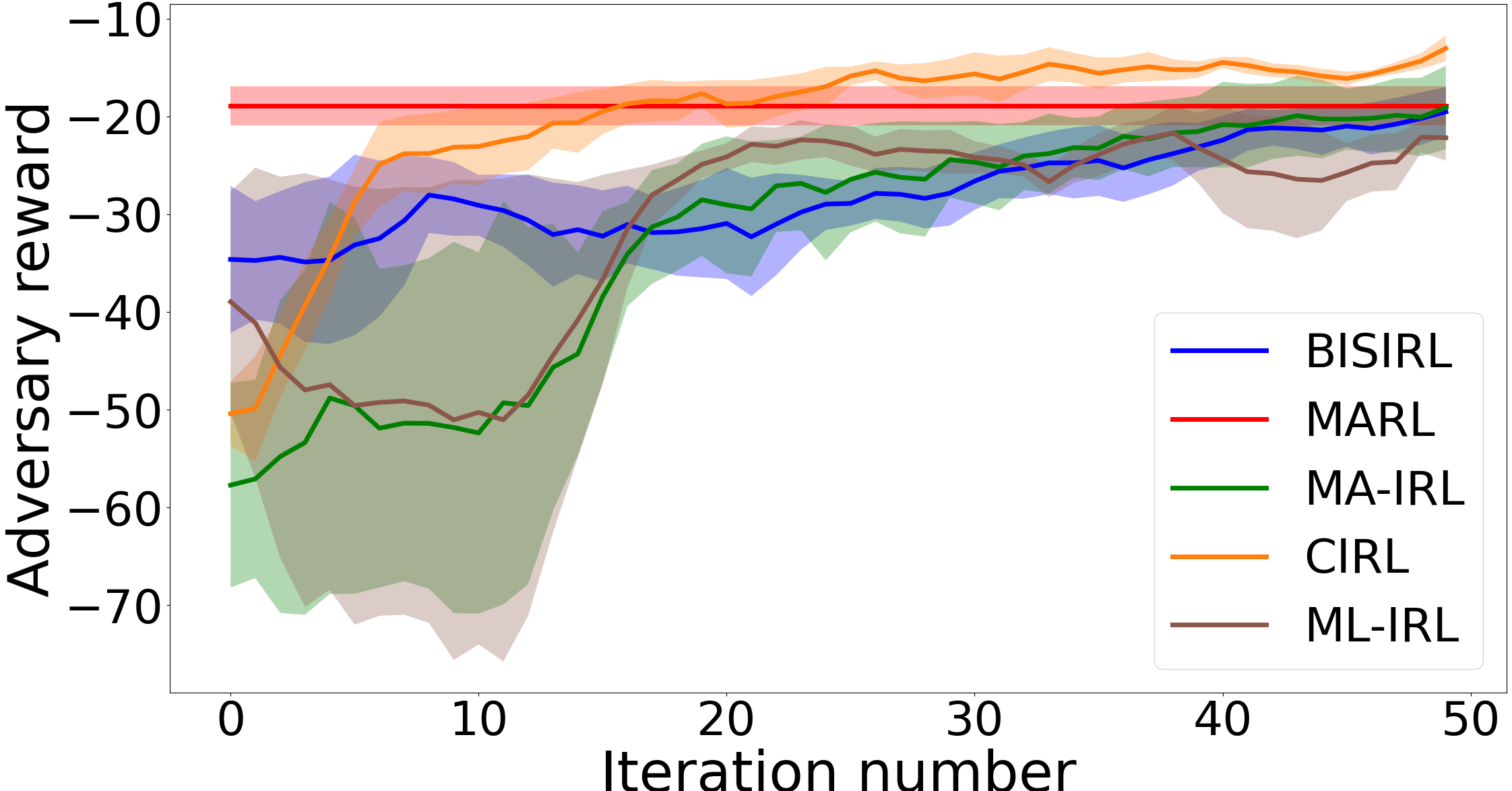}
}
\hspace{-1mm}

\caption{MPE simulation results. \textbf{Left}: Agent's reward. \textbf{Right}: Adversary's reward. The horizontal lines show the cumulative rewards from a MARL method that has access to the ground-truth reward functions after convergence. In contrast, other methods learn the reward functions from demonstrations and interactions, respectively. In each iteration, we compute the policies of both the adversary and the agents via MARL using the currently learned reward functions. We then evaluate these policies using the ground-truth reward functions.}
\label{mpeexp}

\end{figure*}
\begin{figure*}[htb]
\centering

\subfigure{ 
\label{smac_g}
\includegraphics[width=4.25cm]{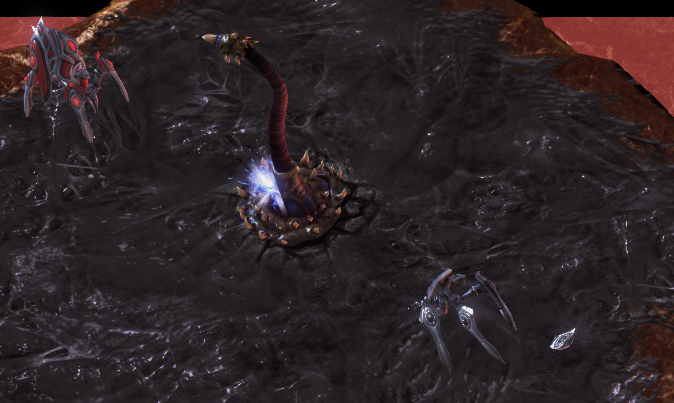}

}
\hspace{-1mm}
\subfigure{ 
\label{smac_cr}
\includegraphics[width=4.25cm]{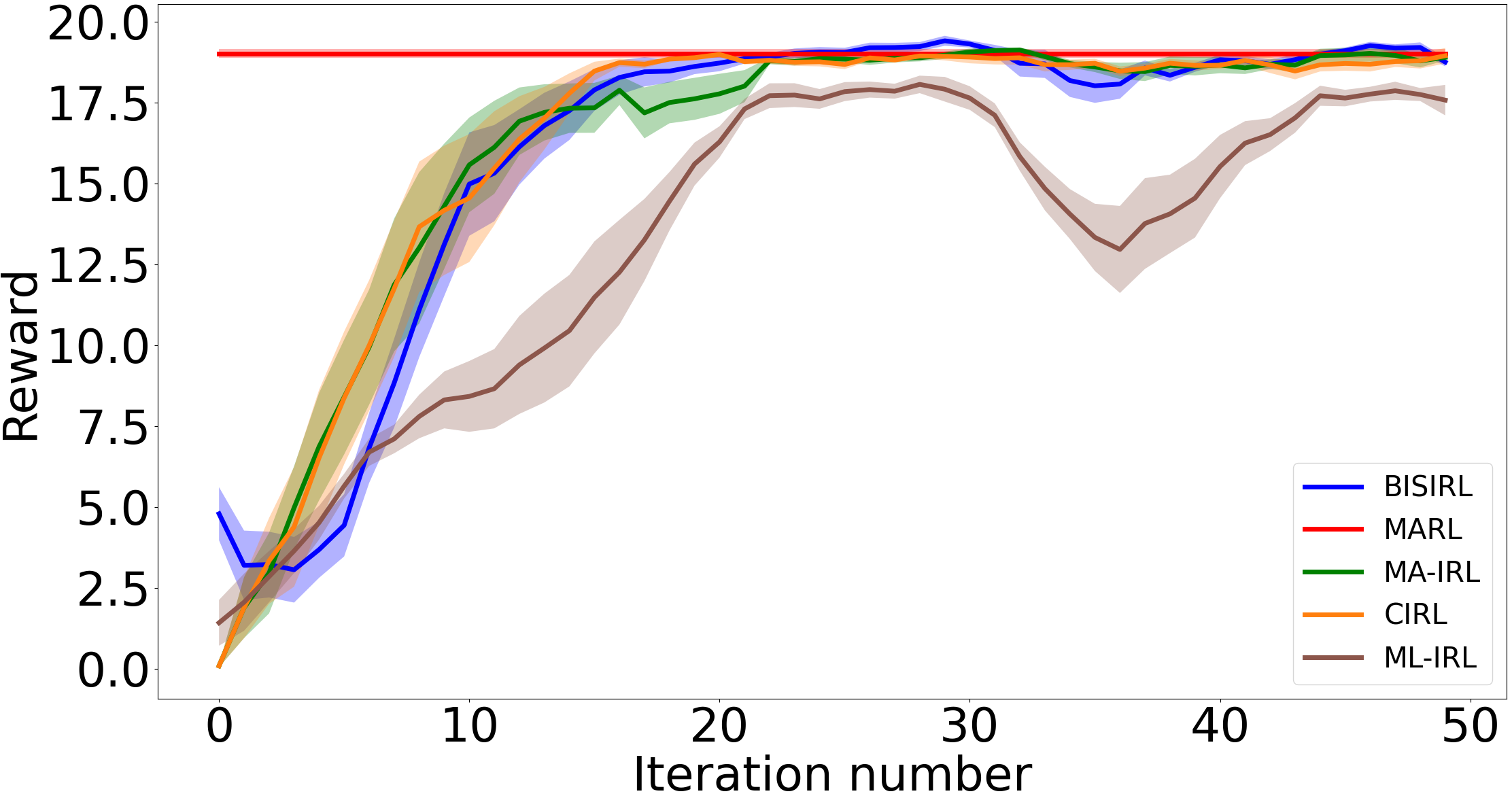}

}
\hspace{-1mm}
\subfigure{ 
\label{smac_win}
\includegraphics[width=4.25cm]{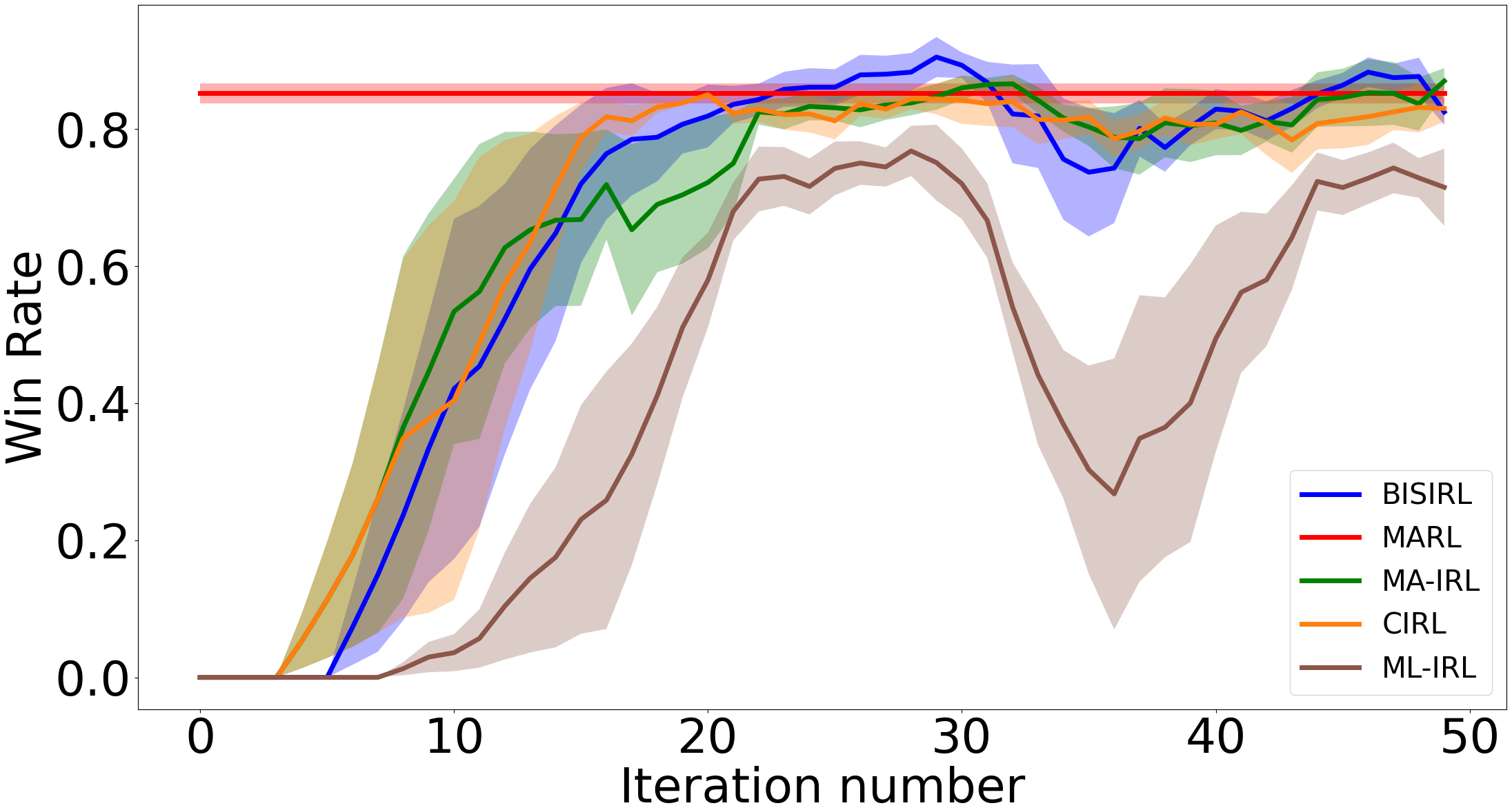}

}
\hspace{-1mm}

\caption{SMAC simulation results. \textbf{Left}: A screenshot of the game. \textbf{Middle}: Cumulative reward. \textbf{Right}: Win rate. The cumulative rewards are computed in the same manner as described in Figure \ref{mpeexp}. Similarly, the win rates are computed based on the policies corresponding to the learned reward functions. The agents are considered victorious when they defeat the enemy. }
\label{smac_result}

 \end{figure*}
\subsection{Multi-Agent Particle Environments}

We use the physical deception environment in the MPE for our simulation. This environment includes one adversary, two cooperating agents, and two landmarks. One of the landmarks is the target landmark. The cooperating agents know which landmark is the target, whereas the adversary does not. The cooperating agents aim to have one of them approach the target as closely as possible while preventing the adversary from reaching it. Conversely, the adversary aims to identify and reach the target landmark. We treat the two cooperating agents as a single learner and the adversary as the expert. Both the learner and the expert have continuous state and action spaces. The learner’s state space is $10$-dimensional and the expert’s is $8$-dimensional, and each has a $5$-dimensional action space. Supplementary experimental details are provided in Appendix Section \ref{mpesetup}.

Because the two agents cooperate and share a common reward, we plot a single cumulative reward curve for both. Figure~\ref{mpeexp} presents results for two randomly initialized reward functions updated over 50 iterations. Initially, both the adversary and the agents act randomly, and the cumulative rewards differ significantly from those of the MARL baseline. As more data are collected, the agents gradually learn their own goal and the adversary’s goal, leading the cumulative rewards of BISIRL and MA-IRL to converge to that of MARL. In contrast, CIRL and ML-IRL fail to do so.
\subsection{StarCraft Multi-Agent Challenge}
We use the scenario 2s\_vs\_1sc from the SMAC  for our experiment. In this scenario, two agents cooperate to defeat a single, more powerful enemy unit (controlled by a built-in AI). Because the environment does not expose the enemy’s state or action information, the enemy’s reward function cannot be learned directly. We therefore designate one agent as the expert (with access to the ground-truth reward function and the optimal cooperative strategy to defeat the enemy) and the other agent as the learner. The learner’s goal is to infer the expert’s strategy and learn to cooperate with the expert accordingly. Both agents have a continuous $17$-dimensional state space and a discrete action space of size $7$. Additional implementation details are in Appendix \ref{smacexp}.

Because the two agents share the same reward, we report a single cumulative reward curve. The middle panel of Figure \ref{smac_result} shows that the cumulative rewards of BISIRL and MA-IRL converge to that achieved by MARL, mirroring the trend observed in Figure \ref{mpeexp}. Since this scenario is fully cooperative, the cumulative reward of CIRL also converges to that of the MARL, whereas ML-IRL does not. We further evaluate performance using the agents’ win rate. At each iteration, we obtain the agents’ policies via MARL with the current learned reward functions. Each learned policy is then evaluated over $100$ independent trials, and the win rate (the fraction of trials in which the agents defeat the enemy) is recorded. The right panel of Figure \ref{smac_result} shows that the win rates of BISIRL, MA-IRL, and CIRL converge to that of the MARL baseline, while ML-IRL fails to do so.
\begin{table*}
\caption{The cumulative reward comparison. \textbf{Top}: The cumulative reward of the learner. \textbf{Bottom}: The cumulative reward of the expert. MARL uses ground-truth reward functions. MA-IRL, BISIRL, CIRL, and ML-IRL use learned reward functions from the last iteration.}
\label{table}
\vskip 0.15in
\begin{center}
\begin{small}
\begin{sc}
\begin{tabular}{p{1.2cm} p{2cm} p{2cm} p{2cm}p{2cm}p{2.2cm}}
\toprule
Learner &  MARL & MA-IRL & BISIRL &CIRL &ML-IRL\\
\midrule
MPE    & $5.03\pm 1.66$& $5.12 \pm 4.08$& $4.84\pm1.81$ &$7.59\pm0.53$&$13.09\pm0.25$\\
SMAC & $19.01 \pm 0.13$& $18.89 \pm 0.20$& $18.87 \pm 0.27$&$18.94\pm0.48$&$17.56\pm0.94$\\
HRI    & $-20.22\pm 3.34$& $-20.63 \pm 1.07$&  $-21.09 \pm 1.87$&$-54.20\pm4.61$ &$-101.45\pm21.22$       \\
Security   & $-16.96\pm 1.45$& $-17.08\pm3.21$& $-17.38 \pm3.59$ & $-21.68\pm0.84$ &$-20.13\pm0.13$\\

\midrule
Expert &  MARL & MA-IRL & BISIRL&CIRL &ML-IRL \\
\midrule
MPE     & $-18.91\pm2.01$& $-20.37\pm4.27$& $-19.53 \pm 2.50$&$-13.05\pm1.32$&$-22.19\pm2.32$\\
SMAC     & $19.01 \pm 0.13$& $18.89 \pm 0.20$& $18.87 \pm 0.27$&$18.94\pm0.48$&$17.56\pm0.94$ \\
HRI    & $-12.98 \pm 4.78$& $-13.52\pm 2.57$& $-13.16\pm 1.03$ &$-38.09\pm0.13$ &$-41.40\pm2.85$\\
Security     & $6.96 \pm 1.45$& $7.03\pm 3.21$&  $6.76\pm 3.59$ &$11.68\pm0.84$ &$9.12\pm0.13$      \\

\bottomrule
\end{tabular}
\end{sc}
\end{small}
\end{center}
\label{cr table}
\vskip -0.1in

\end{table*}

\subsection{Results Analysis}
Figures \ref{mpeexp} and \ref{smac_result} show that the cumulative rewards of MA-IRL and BISIRL converge to values close to those of MARL, whereas CIRL fails to do so in mixed cooperative-competitive environments and ML-IRL fails across all environments. As summarized in Table \ref{cr table}, the final cumulative rewards of BISIRL are within $5\%$ of those achieved by MA-IRL and MARL. This comparison is meaningful because MARL serves as an empirical upper-bound baseline with access to the ground-truth reward functions, while MA-IRL uses demonstrations generated from the ground-truth reward functions of both the learner and the expert. In contrast, BISIRL relies only on the expert’s interactions with the environment and the corresponding values received during these interactions. These results show that BISIRL achieves its intended goal: matching the performance of MARL and MA-IRL under the more restrictive IIRL setting, rather than outperforming these privileged-information baselines.


\section{Conclusion}

We develop a general IIRL framework that enables a learner to infer an expert’s reward function while simultaneously learning an appropriate interaction policy through engagement with the expert. This framework relaxes the restrictive assumption that the learner’s and expert’s reward functions are identical or opposite, making it applicable to more general interaction scenarios where agents may have distinct and potentially misaligned objectives. To solve this problem, we propose the BISIRL algorithm and provide theoretical guarantees on its convergence rate. Empirical results in both continuous and discrete environments demonstrate that BISIRL can accurately recover the behavior of the expert and learn a high-performing interaction policy. These results further show that BISIRL can approach the performance of privileged-information baselines while operating under a more restrictive and realistic IIRL setting.

\bibliographystyle{plainnat}
\bibliography{conference}
\appendix
\section{Notion and notations}
 Define $f(\theta_l, \theta_e) \triangleq J_{l}(\pi_{\theta_l, \theta_e})$, where $J_{l}(\pi_{\theta_l, \theta_e}) \triangleq E^{\pi_{\theta_l, \theta_e}}[\sum_{h=0}^{H-1} \gamma^h r_{\theta_l}(s^h,a^h)]$ is the cumulative reward of the learner with respect to $r_{\theta_l}$. Define the reward gradient expectation of the expert from the state-action pair $(s,a)$ as $\mu_e(s,a) \triangleq E^{\pi_{\theta_l, \theta_e}}[\sum_{h=0}^{H-1} \gamma^h \nabla_{\theta_e} r_{\theta_e}(s^h,a^h)| s^0 = s,a^0 =a]$, the reward gradient expectation of the expert from the state $s$ as $\mu_e(s) \triangleq \int_{a_l \in A_l} \int_{a_e \in A_e} \mu_e(s,a)da_lda_e, \mu_e(s,a)\triangleq E^{\pi_{\theta_l, \theta_e}}[\sum_{h=0}^{H-1} \gamma^h \nabla_{\theta_e} r_{\theta_e}(s^h,a^h)| s^0 = s] $. Analogously, define the reward gradient expectation of the learner as $\mu_l(\pi_{\theta_l, \theta_e}) \triangleq E^{\pi_{\theta_l, \theta_e}}[\sum^{H-1}_{h=0} \gamma^h \nabla_{\theta_l} r_{\theta_l}(s^{h},a^{h})]$, the reward gradient expectation of the learner from the state-action pair $(s,a)$ as $\mu_l(s,a) \triangleq E^{\pi_{\theta_l, \theta_e}}[\sum_{h=0}^{H-1} \gamma^h \nabla_{\theta_l} r_{\theta_l}(s^h,a^h)| s^0 = s,a^0 =a]$, the reward gradient expectation of the learner from the state $s$ as $\mu_l(s) \triangleq \int_{a_l \in A_l} \int_{a_e \in A_e} \mu_l(s,a)da_lda_e \triangleq E^{\pi_{\theta_l, \theta_e}}[\sum_{h=0}^{H-1} \gamma^h \nabla_{\theta_l} r_{\theta_l}(s^h,a^h)| s^0 = s]$. Define the cumulative reward of the learner with respect to $r_{\theta_l}$ from the state-action pair $(s,a)$ as $J_{l}(s,a) \triangleq E^{\pi_{\theta_l, \theta_e}}[\sum_{h=0}^{H-1} \gamma^h r_{\theta_l}(s^h,a^h)|s^0 = s,a^0 =a]$ and define the cumulative reward of the
learner with respect to $r_{\theta_l}$ from the state $s$ as $ J_{l}(s) \triangleq E^{\pi_{\theta_l, \theta_e}}[\sum_{h=0}^{H-1} \gamma^h r_{\theta_l}(s^h,a^h)|s^0 = s]$. Define $P_0(s)$ as the probability distribution of the initial state.

During the proofs, symbols $(i)$ to $(viii)$ are used to represent what theorems or methods are used to get the current step. The symbol $(i)$ represents chain rule, the symbol $(ii)$ represents the linearity of expectation, the symbol $(iii)$ represents the closed-form of geometric series, the symbol $(iv)$ represents the triangle inequality, the symbol $(v)$ represents the hölder’s inequality, the symbol $(vi)$ represents Taylor’s theorem, the symbol $(vii)$ means the usage of other equations in this paper, and the symbol $(viii)$ means keeping expansion.
\section{Fundamental result for policy}
This section lists the expressions for important gradients in continuous state-action space. 
Based on the idea of the soft Q learning \cite{haarnoja2017reinforcement}, we can get:
\begin{equation}\label{softq}
Q^{soft}_{\theta_l,\theta_e} (s,a) =  r_{\theta_l}(s,a)+ r_{\theta_e}(s,a) + \gamma \int_{s'\in S} P(s'| s,a) V^{soft}(s') ds',   
\end{equation}
\begin{equation}\label{softv}
    V^{soft}_{\theta_l,\theta_e}(s) = \ln \int_{a_l \in A_l} \int_{a_e \in A_e} \exp (Q^{soft} (s,a)) da_l da_e,
\end{equation}
\begin{equation}\label{softpi}
    \pi_{\theta_l, \theta_e}(a_l,a_e|s) = \frac{\exp (Q^{soft} (s,a))}{\exp (V^{soft}(s))}.
\end{equation}


As the lower-level optimization problem is a maximum likelihood problem, it is important to know $\nabla _ {\theta_e} \ln (\pi_{\theta_l, \theta_e})$ for the SGD of the lower-level optimization problem. The process for getting $\nabla _ {\theta_e} \ln (\pi_{\theta_l, \theta_e})$ is shown below.\\
From the expression of $\pi_{\theta_l, \theta_e}$, we can get $\nabla _ {\theta_e}\ln (\pi_{\theta_l, \theta_e}) = \nabla _ {\theta_e}Q^{soft} (s,a) -\nabla _ {\theta_e} V^{soft}(s)$, therefore, we can calculate $\nabla _ {\theta_e}Q^{soft} (s,a)$ and $\nabla _ {\theta_e} V^{soft}(s)$ separately.\\
Based on the equation of $V^{soft}(s)$, the gradient $\nabla _ {\theta_e}V^{soft}(s)$ could be calculated as follows:

    \begin{align}\label{dsoftv}
        &\nabla _ {\theta_e} V^{soft}(s),\nonumber\\
        &\overset{(i)}{=}\frac{\int_{a_l \in A_l} \int_{a_e \in A_e} \nabla _ {\theta_e}\exp (Q^{soft} (s,a)) da_l da_e}{\int_{a_l \in A_l} \int_{a_e \in A_e} \exp (Q^{soft} (s,a)) da_l da_e},\nonumber\\
        & \overset{(vii)}{=} \frac{\int_{a_l \in A_l} \int_{a_e \in A_e} \exp (Q^{soft} (s,a)) \nabla _ {\theta_e} Q^{soft} (s,a)da_l da_e}{\exp (V^{soft}(s))},\nonumber\\
        & \overset{(vii)}{=} \int_{a_l \in A_l} \int_{a_e \in A_e} \pi_{\theta_l, \theta_e} (a_l,a_e|s)(\nabla_{\theta_e} r_{\theta_e}(s,a)\nonumber\\
        &+ \gamma \int_{s'\in S} P(s'| s,a) \nabla _ {\theta_e}V^{soft}(s') ds') da_l da_e,\nonumber\\
        &\overset{(viii)}{=}\int_{a_l \in A_l} \int_{a_e \in A_e} \pi_{\theta_l, \theta_e} (a_l,a_e|s)(\nabla_{\theta_e} r_{\theta_e}(s,a)
        + \gamma \int_{s'\in S} P(s'| s,a)\nonumber\\ 
        &(\int_{a_l' \in a_l} \int_{a_e' \in a_e} \pi_{\theta_l, \theta_e} (a_l',a_e'|s')
        [\nabla_{\theta_e} r_{\theta_e}(s',a_l',a_e') \nonumber\\
        &+ \gamma \int_{s''\in S} P(s''| s',a_l',a_e') \nabla _ {\theta_e}V^{soft}(s'') ds''] da_l'da_e' ds') da_l da_e,\nonumber\\
        &=  E^{\pi_{\theta_l, \theta_e}}[\sum_{h=0}^{H-1} \gamma^h \nabla_{\theta_e} r_{\theta_e}(s^h,a^h)| s^0 = s],\nonumber\\
        &= \mu_e(s),
    \end{align}
where the first $(vii)$ uses the expression in equation (\ref{softv}) and (\ref{softpi}), and the second $(iii)$ uses the expression in equation (\ref{softq}).\\
Similarly, we can get $\nabla _ {\theta_e} Q^{soft}(s,a)$ from $Q^{soft}(s,a)$.\\
\begin{equation}\label{dsoftq}
    \begin{aligned}
        &\nabla _ {\theta_e} Q^{soft}(s,a),\\
        &\overset{(vii)}{=} \nabla_{\theta_e} r_{\theta_e}(s,a)
        + \gamma \int_{s'\in S} P(s'| s,a)\nabla _ {\theta_e}V^{soft}(s') ds') da_l da_e,\\        
        &\overset{(viii)}{=}  E^{\pi_{\theta_l, \theta_e}}[\sum_{h=0}^{H-1} \gamma^h \nabla_{\theta_e} r_{\theta_e}(s^h,a^h)| s^0 = s,a^0 =a],\\
        &=\mu_e(s,a),
    \end{aligned}
\end{equation}
where $(vii)$ uses the expression in equation (\ref{softq}).\\
By summing the results of equation (\ref{dsoftv}) and (\ref{dsoftq}), we can get $\nabla _ {\theta_e} \ln (\pi_{\theta_l, \theta_e})$ as follows:
\begin{equation}\label{dsoftpi}
         \nabla _ {\theta_e} \ln (\pi_{\theta_l, \theta_e}(a_l,a_e|s))
        = \nabla _ {\theta_e} Q^{soft}(s,a) -  \nabla _ {\theta_e} V^{soft}(s)
        = \mu_e(s,a) - \mu_e(s).
\end{equation}
The way to get $\nabla _ {\theta_l} \ln (\pi_{\theta_l, \theta_e})$ is same as that of $\nabla _ {\theta_e} \ln (\pi_{\theta_l, \theta_e})$, and the gradient $\nabla _ {\theta_l} \ln (\pi_{\theta_l, \theta_e}(a_l,a_e|s)) = \mu_l(s,a) - \mu_l(s)$.
\section{The proof of Lemma \ref{lemma1}}\label{plemma1}
In this section, we derived gradients that are necessary for our method.\\

Define $P_D(s^h=s,a_l^h=a_l,a_e^h=a_e) \triangleq \begin{cases}
    1& s^h=s,a_l^h=a_l,a_e^h=a_e\\
    0& \text{otherwise}\\
\end{cases}$
, where $(s,a) \in D_{\theta_l, r_{e}}$, as  the probability of $(s,a)$ occurring at time $h$ in the demonstration set $D_{\theta_l, r_{e}}$. 
With the fundamental result of the policy, we can derive the $\nabla _ {\theta_e} L(\theta_l, \theta_e)$ and prove the Lemma \ref{lemma1} as follows:
    \begin{align*}
        &\nabla _ {\theta_e} L(\theta_l, \theta_e),\\
        &= -\sum_{h=0}^{H-1} \gamma^h\int_{s\in S}\int_{a_l \in A_l} \int_{a_e\in A_e} P_D(s^h=s,a_l^h=a_l,a_e^h=a_e) \nabla_{\theta_e} \ln (\pi_{\theta_l, \theta_e})da_eda_lds + \lambda \theta_e,\\
        &\overset{(vii)}{=} -\sum_{h=0}^{H-1} \gamma^h\int_{s\in S}\int_{a_l \in A_l} \int_{a_e\in A_e} P_D(s^h=s,a_l^h=a_l,a_e^h=a_e) (\nabla_{\theta_e} r_{\theta_e}(s,a)\\
        &+E^{\pi_{\theta_l, \theta_e}}[\sum_{t' =t+1}^{T-1}\gamma^{t'-t-1}\nabla_{\theta_e} r_{\theta_e}(s^h,a^h)|s^h=s,a_l^h=a_l,a_e^h=a_e]\\
        &-E^{\pi_{\theta_l, \theta_e}}[\sum_{t' =t}^{T-1}\gamma^{t'-t}\nabla_{\theta_e} r_{\theta_e}(s^h,a^h)|s^h=s])da_eda_lds + \lambda \theta_e,\\
        &\overset{(ii)}{=}  -\sum_{h=0}^{H-1} \gamma^h\int_{s\in S}\int_{a_l \in A_l} \int_{a_e\in A_e}P_D(s^h=s,a_l^h=a_l,a_e^h=a_e) (\nabla_{\theta_e} r_{\theta_e}(s,a) da_eda_lds\\
        &- \int_{s\in S}\int_{a_l \in A_l} \int_{a_e\in A_e}P_D(s^0=s,a_l^0=a_l,a_e^0=a_e)\\
        &E^{\pi_{\theta_l, \theta_e}}[\sum_{t =0}^{T-1}\gamma^{t}\nabla_{\theta_e} r_{\theta_e}(s^h,a^h)|s^0=s])da_eda_lds+ \lambda \theta_e,\\
        &=\mu_e(\pi_{\theta_l, \theta_e})-\hat{\mu}_e(D_{\theta_l, r_{e}})+ \lambda \theta_e,
    \end{align*}
where $(vii)$ uses the expression of equation (\ref{dsoftpi}).\\
\section{Other needed gradients}\label{other g}
\begin{lemma}\label{gradients}
    Suppose Assumption \ref{assump1} holds, the first-order gradients $\nabla_{\theta_l}L(\theta_l,\theta_e)=\mu_l(\pi_{\theta_l, \theta_e})-\hat{\mu}_l(D_{\theta_l, r_{e}}), \nabla_{\theta_l} f(\theta_l,\theta_e)=E^{\pi_{\theta_l, \theta_e}}[\sum^{H-1}_{h=0} \gamma^h[(\mu_l(s^h,a^h) - \mu_l(s^h))J_{l}(s^h,a^h)^T]]+\mu_l(\pi_{\theta_l, \theta_e})$, $\nabla_{\theta_e}f(\theta_l,\theta_e) = E^{\pi_{\theta_l, \theta_e}}[\sum^{H-1}_{h=0} \gamma^h[(\mu_e(s^h,a^h) - \mu_e(s^h))(J_{l}(s^h,a^h))^T]]$. The second-order gradients $\nabla^2_{\theta_l\theta_e}L(\theta_l, \theta_e) =E^{\pi_{\theta_l, \theta_e}}[\sum^{H-1}_{h=0} \gamma^h(\mu_e(s,a) - \mu_e(s))\mu_l(s,a)^T]$, $\nabla^2_{\theta_e\theta_e}L(\theta_l, \theta_e)=E^{\pi_{\theta_l, \theta_e}}[\sum^{H-1}_{h=0} \gamma^h[(\mu_e(s^h,a^h)- \mu_e(s^h))\mu_e(s^h,a^h)^T+\nabla_{\theta_e}^2 r_{\theta_e}(s^h,a^h)]]-\frac{1}{d}\sum_{i=0}^d\sum_{h=0}^{H-1} \gamma^h \nabla_{\theta_e} r_{\theta_e}(s^{ih}, a^{ih}) + \lambda$    
\end{lemma}
\begin{proof}
    Through the same process of calculating $\nabla_{\theta_e}L(\theta_l,\theta_e)$, we can get the gradient $\nabla_{\theta_l}L(\theta_l,\theta_e)$ following the same process in Lemma \ref{lemma1}, where $\mu_l(\pi_{\theta_l, \theta_e}) \triangleq E^{\pi_{\theta_l, \theta_e}}[\sum^{H-1}_{h=0} \gamma^h \nabla_{\theta_l} r_{\theta_l}(s^{h},a^{h})]$.\\
    
    The process for getting $\nabla_{\theta_l} f(\theta_l,\theta_e)$ is shown below.\\
From the equation of $ f(\theta_l,\theta_e)$, we can see that $\nabla_{\theta_l} f(\theta_l,\theta_e) = \nabla_{\theta_l} J_{l}(\pi_{\theta_l, \theta_e})$, then $\nabla_{\theta_l} J_{l}(\pi_{\theta_l, \theta_e})$ can be calculated.\\
The calculation of deriving $\nabla_{\theta_l} J_{l}(\pi_{\theta_l, \theta_e})$ is as follows:
\begin{equation*}
    \begin{aligned}
        &\nabla_{\theta_l} J_{l}(\pi_{\theta_l, \theta_e}), \\
        &\overset{(i)}{=} \int_{s^0 \in S }P_0(s^0) \int_{a_l^0 \in a_l} \int_{a_e^0 \in a_e} [\nabla_{\theta_l}\pi(a_l^0,a_e^0|s^0) J_{l}(s^0,a_l^0,a_e^0) ^T\\
        &+ \pi(a_l^0,a_e^0|s^0)\nabla_{\theta_l}J_{l}(s^0,a_l^0,a_e^0)^T]da_l^0da_e^0ds^0,\\
        &= \int_{s^0 \in S }P_0(s^0) \int_{a_l^0 \in a_l} \int_{a_e^0 \in a_e} [\nabla_{\theta_l}\pi(a_l^0,a_e^0|s^0) J_{l}(s^0,a_l^0,a_e^0)^T
        +\pi(a_l^0,a_e^0|s^0)(\nabla_{\theta_l} r_{\theta_l}(s^0,a_l^0,a_e^0)\\
        &+\gamma \int_{s^1 \in S }P(s^1| s^0,a_l^0,a_e^0)\nabla_{\theta_l}J_{l}(s^1))ds^1 ]da_l^0da_e^0ds^0,\\
        & \overset{(viii)}{=} \int_{s^0 \in S }P_0(s^0) \int_{a_l^0 \in a_l} \int_{a_e^0 \in a_e} \{\nabla_{\theta_l}\pi(a_l^0,a_e^0|s^0) J_{l}(s^0,a_l^0,a_e^0)^T
        +\pi(a_l^0,a_e^0|s^0)[\nabla_{\theta_l} r_{\theta_l}(s^0,a_l^0,a_e^0)\\
        &+\gamma \int_{s^1 \in S }P(s^1| s^0,a_l^0,a_e^0)\int_{a_l^1 \in a_l} \int_{a_e^1 \in a_e}\nabla_{\theta_l}\pi(a_l^1,a_e^1|s^1) J_{l}(s^1,a_l^1,a_e^1)^T\\
        &+\pi(a_l^1,a_e^1|s^1)(\nabla_{\theta_l} r_{\theta_l}(s^1,a_l^1,a_e^1)
        +\gamma \int_{s^2 \in S }P(s^2| s^1,a_l^1,a_e^1)\nabla_{\theta_l}J_{l}(s^2))ds^2 ]ds^1 \}da_l^0da_e^0ds^0,\\
        &= E^{\pi_{\theta_l, \theta_e}}[\sum^{H-1}_{h=0} \gamma^h(\frac{\nabla_{\theta_l}\pi(a_l^h,a_e^h|s^h)}{\pi(a_l^h,a_e^h|s^h)}J_{l}(s^h,a^h)^T+ \nabla_{\theta_l} r_{\theta_l}(s^h,a^h))],\\
        &= E^{\pi_{\theta_l, \theta_e}}[\sum^{H-1}_{h=0} \gamma^h(\nabla_{\theta_l} \ln (\pi_{\theta_l, \theta_e})J_{l}(s^h,a^h)^T + \nabla_{\theta_l} r_{\theta_l}(s^h,a^h))],\\
        & \overset{(vii)}{=} E^{\pi_{\theta_l, \theta_e}}[\sum^{H-1}_{h=0} \gamma^h[(\mu_l(s^h,a^h) - \mu_l(s^h))J_{l}(s^h,a^h)^T]]+\mu_l(\pi_{\theta_l, \theta_e}),
    \end{aligned}
\end{equation*}
where $(vii)$ use the expression of the equation (\ref{dsoftpi}).\\ 

The $\nabla_{\theta_e} f(\theta_l,\theta_e)$ is calculated in the same way as the  $\nabla_{\theta_l}f(\theta_l,\theta_e)$.\\

The process for getting $\nabla^2_{\theta_e}L(\theta_l, \theta_e)$ is shown below. \\
As we proved in Lemma \ref{lemma1}, the gradient $\nabla_{\theta_e} L(\theta_l, \theta_e) = \mu_e(\pi_{\theta_l, \theta_e})-\hat{\mu}_e(D_{\theta_l, r_{e}}) + \lambda \theta_e$, as a result, we can take the derivative of each term separately.\\
The derivative of $\mu_e(\pi_{\theta_l, \theta_e})$ w.r.t $\theta_e$ is calculated as follows:
    \begin{align*}
        &\nabla_{\theta_e} \mu_e(\pi_{\theta_l, \theta_e}),\\
        &\overset{(i)}{=} \int_{s^0 \in S }P_0(s^0) \int_{a_l^0 \in a_l} \int_{a_e^0 \in a_e} [\nabla_{\theta_e}\pi(a_l^0,a_e^0|s^0) \mu_e(s^0,a_l^0,a_e^0) ^T\\
        &+ \pi(a_l^0,a_e^0|s^0)\nabla_{\theta_e}\mu_e(s^0,a_l^0,a_e^0)^T]da_l^0da_e^0ds^0,\\
        &= \int_{s^0 \in S }P_0(s^0) \int_{a_l^0 \in a_l} \int_{a_e^0 \in a_e} [\nabla_{\theta_l}\pi(a_l^0,a_e^0|s^0) \mu_e(s^0,a_l^0,a_e^0)^T\\
        &+\pi(a_l^0,a_e^0|s^0)(\nabla_{\theta_e}^2 r_{\theta_e}(s^0,a_l^0,a_e^0)
        +\gamma \int_{s^1 \in S }P(s^1| s^0,a_l^0,a_e^0)\nabla_{\theta_e}\mu_e(s^1))ds^1 ]da_l^0da_e^0ds^0,\\
        & \overset{(viii)}{=} \int_{s^0 \in S }P_0(s^0) \int_{a_l^0 \in a_l} \int_{a_e^0 \in a_e} \{\nabla_{\theta_l}\pi(a_l^0,a_e^0|s^0) \mu_e(s^0,a_l^0,a_e^0)^T\\
        &+\pi(a_l^0,a_e^0|s^0)[\nabla_{\theta_e}^2 r_{\theta_e}(s^0,a_l^0,a_e^0)\\
        &+\gamma \int_{s^1 \in S }P(s^1| s^0,a_l^0,a_e^0)\int_{a_l^1 \in a_l} \int_{a_e^1 \in a_e}\nabla_{\theta_e}\pi(a_l^1,a_e^1|s^1) \mu_e(s^1,a_l^1,a_e^1)^T\\
        &+\pi(a_l^1,a_e^1|s^1)(\nabla_{\theta_e}^2 r_{\theta_e}(s^1,a_l^1,a_e^1)
        +\gamma \int_{s^2 \in S }P(s^2| s^1,a_l^1,a_e^1)\nabla_{\theta_e}\mu_e(s^2))ds^2 ]ds^1 \}da_l^0da_e^0ds^0,\\
        &= E^{\pi_{\theta_l, \theta_e}}[\sum^{H-1}_{h=0} \gamma^h(\frac{\nabla_{\theta_e}\pi(a_l^h,a_e^h|s^h)}{\pi(a_l^h,a_e^h|s^h)}\mu_e(s^h,a^h)^T+ \nabla_{\theta_e}^2 r_{\theta_e}(s^h,a^h))],\\
        &= E^{\pi_{\theta_l, \theta_e}}[\sum^{H-1}_{h=0} \gamma^h(\nabla_{\theta_e} \ln (\pi_{\theta_l, \theta_e})\mu_e(s^h,a^h)^T + \nabla_{\theta_e}^2 r_{\theta_e}(s^h,a^h))],\\
        & \overset{(vii)}{=} E^{\pi_{\theta_l, \theta_e}}[\sum^{H-1}_{h=0} \gamma^h[(\mu_e(s^h,a^h) - \mu_e(s^h))\mu_e(s^h,a^h)^T+\nabla_{\theta_e}^2 r_{\theta_e}(s^h,a^h)]],
    \end{align*}
where $(vii)$ use the expression of the equation (\ref{dsoftpi}).\\ 
With the result of $\nabla_{\theta_e} \mu_e(\pi_{\theta_l, \theta_e})$, the derivative of $\nabla_{\theta_e} L(\theta_l, \theta_e)$ w.r.t $\theta_e$ is as follows:
\begin{equation*}
    \begin{aligned}
        &\nabla^2_{\theta_e}L(\theta_l, \theta_e) = \nabla_{\theta_e} (\nabla_{\theta_e} L(\theta_l, \theta_e)),\\
         &= \nabla_{\theta_e}(\mu_e(\pi_{\theta_l, \theta_e})-\hat{\mu}_e(D_{\theta_l, r_{e}}) + \lambda \theta_e) ,\\
         & \overset{(ii)}{=} E^{\pi_{\theta_l, \theta_e}}[\sum^{H-1}_{h=0} \gamma^h[(\mu_e(s^h,a^h)- \mu_e(s^h))\mu_e(s^h,a^h)^T+\nabla_{\theta_e}^2 r_{\theta_e}(s^h,a^h)]]\\
         &-\frac{1}{d}\sum_{i=0}^d\sum_{h=0}^{H-1} \gamma^h \nabla_{\theta_e}^2 r_{\theta_e}(s^{ih}, a^{ih}) + \lambda.
    \end{aligned}
\end{equation*}

Through the same process of calculating $\nabla_{\theta_e} \mu_e(\pi_{\theta_l, \theta_e})$, the result is as follows:
\begin{equation*}
    \nabla^2_{\theta_l\theta_e}L(\theta_l, \theta_e) = E^{\pi_{\theta_l, \theta_e}}[\sum^{H-1}_{h=0} \gamma^h(\mu_e(s,a) - \mu_e(s))\mu_l(s,a)^T]
\end{equation*}
\end{proof}
\section{Properties of the lower level optimization problem}
\begin{lemma}\label{low}
    Suppose Assumption \ref{assump1} holds, for any $\theta_l \in \mathbb{R}^n$ and $\theta_e \in\mathbb{R}^{m}$, $L$ is continuously twice differentiable in $(\theta_l,\theta_e)$.\\
    For any $\bar{\theta}_1 \in \mathbb{R}^n$, $\nabla_{\theta_e} L(\bar{\theta}_1,\theta_e)$ is Lipschitz continuous ($w.r.t$ $\theta_e$) with constant $L_{L_{\theta_e}} > 0$. \\
 For any $\bar{\theta}_1 \in \mathbb{R}^n$ and $\bar{\theta}_2 \in\mathbb{R}^{m}$, we have $\|\nabla_{\theta_l\theta_e}^2 L(\bar{\theta}_1,\bar{\theta}_2)\| \leq C_{L_{\theta_l\theta_e}}$ for some constant $C_{L_{\theta_l\theta_e}} > 0$.\\
 For any $\bar{\theta}_1 \in \mathbb{R}^n$, $\nabla_{\theta_l\theta_e}^2 L(\bar{\theta}_1,\theta_e)$ and $\nabla_{\theta_e\theta_e}^2 L(\bar{\theta}_1,\theta_e)$ are Lipschitz continuous ($w.r.t$ $\theta_e$) with constants $L_{L_{\theta_l\theta_e}} > 0$ and $L_{L_{\theta_e\theta_e}}>0$. \\
 For any $\bar{\theta}_2 \in\mathbb{R}^{m}$, $\nabla_{\theta_l\theta_e}^2 L(\theta_l,\bar{\theta}_2)$ and $\nabla_{\theta_e\theta_e}^2 L(\theta_l,\bar{\theta}_2)$ are Lipschitz continuous ($w.r.t$ $\theta_l$) with constants $\bar{L}_{L_{\theta_l\theta_e}} > 0$ and $\bar{L}_{L_{\theta_e\theta_e}} > 0$\\
 
\end{lemma}
\begin{proof}
    Suppose that $h_1$ is a real-valued function defined and differentiable on an interval $H_1 \subset \mathbb{R}^n$. If $\|\nabla h_1\|$ is bounded, then $h_1$ is a Lipschitz function on $H_1$. So we need to prove $\nabla^2_{\theta_e\theta_e}L(\theta_l, \theta_e)$ is bounded.  From Assumption \ref{assump1}, we can show that $\exists R_{1g}>0, \|\nabla r_{\theta_l}\| \leq R_{1g}$.\\
   
\begin{equation*}   
        \|\mu_l(s)\|
        \leq  E^{\pi_{\theta_l, \theta_e}}[\sum_{h=0}^{H-1} \gamma^h R_{1g}| s^0 = s]
        \overset{(i)}{\leq}\frac{R_{1g}}{1-\gamma} ,
\end{equation*}
where $(i)$ uses the closed-form of geometric series.\\
As a result, $\|\mu_l(s)\|$ is bounded, proved through the same way, $\|\mu_e(s)\|,\|\mu_l(s,a)\|,\|\mu_e(s,a)\|$ are also bounded. Based on the Lemma \ref{gradients}, all elements of $\nabla^2_{\theta_e\theta_e}L(\theta_l, \theta_e)$ are finite, therefore, $\|\nabla^2_{\theta_e\theta_e}L(\theta_l, \theta_e)\|$ is bounded and the $\nabla_{\theta_e} L(\bar{\theta}_1,\theta_e)$ is Lipschitz continuous.\\
In the same way of proving $\|\nabla^2_{\theta_e\theta_e}L(\theta_l, \theta_e)\|$ is bounded, we can show $\|\nabla^2_{\theta_l\theta_e}L(\theta_l, \theta_e)\|$ is bounded.\\

    We need to prove $\nabla_{\theta_e\theta_e\theta_e}^3 L(\theta_l,\theta_e)$ and $\nabla_{\theta_l\theta_e\theta_e}^3 L(\theta_l,\theta_e)$ are bounded. The proof of $\nabla_{\theta_l\theta_e\theta_e}^3 L(\theta_l,\theta_e)$ is bounded as follows:\\
    \begin{equation*}
    \begin{aligned}
        &\nabla_{\theta_l\theta_e\theta_e}^3 L(\theta_l,\theta_e),\\
        &=\nabla_{\theta_e}(E^{\pi_{\theta_l, \theta_e}}[\sum^{H-1}_{h=0} \gamma^h(\mu_e(s,a) - \mu_e(s))\mu_l(s,a)^T]),\\
        & \overset{(i)}{=} E^{\pi_{\theta_l, \theta_e}}[\sum^{H-1}_{h=0} \gamma^h (\nabla_{\theta_e}\mu_e(s,a))\mu_l(s,a)^T) +\mu_e(s,a)(\nabla_{\theta_e}\mu_l(s,a)^T) \\
        &- (\nabla_{\theta_e}\mu_e(s))\mu_l(s,a)^T) - \mu_e(s)(\nabla_{\theta_e}\mu_l(s,a)^T) ].
    \end{aligned}
\end{equation*}
 Each gradient inside the expectation could be derived through the process of deriving $\nabla_{\theta_l} \mu_l(\pi_{\theta_l, \theta_e})$ in the proof of Lemma \ref{gradients} and these gradients are all finite with the same way of proving $\|\nabla^2_{\theta_e\theta_e}L(\theta_l, \theta_e)\|$ is bounded. The third-order gradient $\nabla_{\theta_l\theta_e\theta_e}^3 L(\theta_l,\theta_e)$ is bounded with the smae process. Therefore, the third-order gradients of $L(\theta_l,\theta_e)$ are all bounded.\\
 Through the same procedure, we can prove $\nabla_{\theta_e\theta_e\theta_l}^3 L(\theta_l,\theta_e)$ and $\nabla_{\theta_l\theta_e\theta_l}^3 L(\theta_l,\theta_e)$ are bounded.\\
 At the same time, the  existence of  $\nabla^2_{\theta_l\theta_e}L(\theta_l, \theta_e), \nabla^2_{\theta_e\theta_e}L(\theta_l, \theta_e)$ are shown. Analogously, the existence of $\nabla^2_{\theta_l\theta_l}L(\theta_l, \theta_e), \nabla^2_{\theta_e\theta_l}L(\theta_l, \theta_e)$ could be proved in the same way. The third-order gradients of $L(\theta_l, \theta_e)$ are bounded. Therefore, $L$ is continuously twice differentiable in $(\theta_l,\theta_e)$\\
 
\end{proof}

\section{Properties of the upper-level optimization problem}
\begin{lemma}\label{upperlip}
    Suppose Assumption \ref{assump1} holds, for any $\bar{\theta}_1 \in \mathbb{R}^n$, $\nabla_{\theta_l} f(\bar{\theta}_1;\theta_e)$ and $\nabla_{\theta_e} f(\bar{\theta}_1;\theta_e)$  are Lipschitz continuous ($w.r.t$ $\theta_e$) with constants $L_{f_{\theta_l}} > 0$ and $L_{f_{\theta_e}} > 0$. \\
    For any $\bar{\theta}_2 \in\mathbb{R}^{m}$, $\nabla_{\theta_e} f(\theta_l;\bar{\theta}_2)$ is Lipschitz continuous ($w.r.t$ $\theta_l$) with constants $\bar{L}_{f_{\theta_e}} > 0$\\
    For any $\bar{\theta}_1 \in \mathbb{R}^n$ and $\bar{\theta}_2 \in \mathbb{R}^m$, we have $\|\nabla_{\theta_e} f(\bar{\theta}_1;\bar{\theta}_2)\| \leq C_{f_{\theta_e}}$ for some $C_{f_{\theta_e}} > 0$.\\
\end{lemma}
\begin{proof}
\begin{equation*}
        \|J_{l}(\pi_{\theta_l, \theta_e})\|\leq  E^{\pi_{\theta_l, \theta_e}}[\sum_{h=0}^{H-1} \gamma^h r_{\theta_l}| s^0 = s]\overset{(i)}{\leq}\frac{R_{\theta_l}}{1-\gamma},
\end{equation*}
where $(i)$ uses the closed-form of geometric series and $\|r_{\theta_l}\|\leq R_{\theta_l}$.\\
So the cumulative reward value $J_{l}$ is bounded.\\
\begin{equation*}
    \begin{aligned}
        &\nabla_{\theta_l\theta_e} f(\theta_l;\theta_e),\\
        & = \nabla_{\theta_e} (E^{\pi_{\theta_l, \theta_e}}[\sum^{H-1}_{h=0} \gamma^h[(\mu_l(s^h,a^h) - \mu_l(s^h))(J_{l}(s^h,a^h))^T ]]+ \mu_l(\pi_{\theta_l, \theta_e}),\\
        & \overset{(i)}{=} E^{\pi_{\theta_l, \theta_e}}[\sum^{H-1}_{h=0} \gamma^h [(\nabla_{\theta_e}\mu_l(s^h,a^h) - \nabla_{\theta_e}\mu_l(s^h))(J_{l}(s^h,a^h))^T\\
        &+(\mu_l(s^h,a^h) - \mu_l(s^h))(\nabla_{\theta_e}J_{l}(s^h,a^h))^T] ]+\nabla_{\theta_e}\mu_l(\pi_{\theta_l, \theta_e}).
    \end{aligned}
\end{equation*}

Refer to the proof of Lemma \ref{low}, all elements in $\nabla_{\theta_l\theta_e} f(\theta_l;\theta_e)$ are finite, itself is bounded. Analogously, $\nabla_{\theta_e\theta_e} f(\theta_l;\theta_e)$ is bounded under the same proofing process.   \\
Through the same procedure , we can prove $\nabla_{\theta_e\theta_l} f(\theta_l;\theta_e)$ is bounded.\\
All element of $\nabla_{\theta_e} f(\bar{\theta}_1;\bar{\theta}_2)$ are bounded. Therefore, all elements for $\nabla_{\theta_e} f(\theta_l;\theta_e)$ is bounded and it is bounded.\\
\end{proof}

\section{ Properties of the approximation errors}
Define $b_{12}(k) = \hat{\nabla}^2_{\theta_l\theta_e}L(\theta_l(k), \theta_e(k))-\nabla^2_{\theta_l\theta_e}L(\theta_l(k), \theta_e(k))$, $b_{22}(k)=\hat{\nabla}^2_{\theta_e\theta_e}L(\theta_l(k), \theta_e(k)) - \nabla^2_{\theta_e\theta_e}L(\theta_l(k), \theta_e(k))$, $b_1(k)= \hat{\nabla}_{\theta_l}f(\theta_l(k),\theta_e(k))-\nabla_{\theta_l}f(\theta_l(k),\theta_e(k))$, $b_2 (k)= \hat{\nabla}_{\theta_e}f(\theta_l(k),\theta_e(k))-\nabla_{\theta_e}f(\theta_l(k),\theta_e(k))$, $b(k) = [\hat{\nabla}^2_{\theta_e\theta_e}L(\theta_l(k), \theta_e(k))]^{-1}\hat{\nabla}_{\theta_e}f(\theta_l(k),\theta_e(k)) - [\nabla^2_{\theta_e\theta_e}L(\theta_l(k), \theta_e(k))]^{-1}\nabla_{\theta_e}f(\theta_l(k),\theta_e(k))$ at iteration $k$, $b_a(k)= \hat{\nabla} f(\theta_l(k), \theta_e(k)) - \nabla f(\theta_l(k), \theta_e(k))$ at the iteration $k$. Through the same procedure in the proof of the Lemma \ref{upperlip}, we can get 
conclude that each element of forth-order gradient of $L(\theta_l, \theta_e)$ is bounded by the constant $C_L$ and each element of the third-order gradient of $f(\theta_l(k),\theta_e(k))$ is bounded by $C_f$.\\
\begin{lemma}\label{bias}
    The biases $b_{12}(k)$, $b_{22}(k)$, $b_1(k)$, $b_2(k)$, $b(k)$, and $b_a(k)$ are bounded, 
    \begin{equation*}    
        \|b_{12}(k)\|
        \leq \frac{C_{L}p^2(k)\sqrt{m}}{6}\{[m^3-(m-1)^3]\alpha_l^2+(m-1)^3\alpha_l\alpha_0^3\},   
    \end{equation*}
    \begin{equation*}   
        \|b_{22}(k)\|
        \leq \frac{C_{L}p^2(k)\sqrt{m}}{6}\{[m^3-(m-1)^3]\alpha_l^2+(m-1)^3\alpha_l\alpha_0^3\}    ,
    \end{equation*}
    \begin{equation*}    
        \|b_{1}(k)\|
        \leq \frac{C_{f}p^2(k)\sqrt{n}}{6}\{[n^3-(n-1)^3]\alpha_l^2+(n-1)^3\alpha_l\alpha_0^3\}  ,    
    \end{equation*}
    \begin{equation*}   
        \|b_{2}(k)\|
        \leq \frac{C_{f}p^2(k)\sqrt{m}}{6}\{[m^3-(m-1)^3]\alpha_l^2+(m-1)^3\alpha_l\alpha_0^3\} ,      
    \end{equation*}
    \begin{equation*}
        \|b(k)\| \leq \frac{2C_{f_{\theta_e}}(C_{L}p^2(k)+C_{f}p^2(k))}{6\sqrt{m}\lambda^2}\{[m^3-(m-1)^3]\alpha_l^2+(m-1)^3\alpha_l\alpha_0^3\},
    \end{equation*}
    \begin{equation*}
        \begin{aligned}
            &\|b_a(k)\|\\
                    &\leq \frac{C_{f}p^2(k)\sqrt{n}}{6}\{[n^3-(n-1)^3]\alpha_l^2+(n-1)^3\alpha_l\alpha_0^3\}\\
        &+  \frac{2C_{L_{\theta_l\theta_e}}C_{L}p^2(k)+ 2C_{L_{\theta_l\theta_e}}\sqrt{m}\lambda C_{f}p^2(k)}{6\sqrt{m}\lambda^2}\\
        &\{[m^3-(m-1)^3]\alpha_l^2+(m-1)^3\alpha_l\alpha_0^3\} 
        +\frac{{C_{f_{\theta_e}}}C_{L}p^2(k)}{6m\lambda^3}\{[m^3-(m-1)^3]\alpha_l^2+(m-1)^3\alpha_l\alpha_0^3\}\\        
        &+ \frac{2p^4(k)C_{L} (C_{f_{\theta_e}}C_{L} + \sqrt{m}\lambda C_{f}) }{36\lambda^2}\{[m^3-(m-1)^3]\alpha_l^2
        +(m-1)^3\alpha_l\alpha_0^3\}^2,\\
        \end{aligned}
    \end{equation*}
\end{lemma}
\begin{proof}
    According to the Lemma 1 in \cite{spall1992multivariate}, the approximation error $b_{12l}(k)$ for $\hat{\nabla}^2_{\theta_l\theta_e}L(\theta_l, \theta_e)$ is :
    \begin{equation*}
    b_{12l}(k) \leq \frac{C_{L}p^2(k)}{6}\{[m^3-(m-1)^3]\alpha_l^2+(m-1)^3\alpha_l\alpha_0^3\},
\end{equation*}
where $b_{12l}(k)$ represent the $l$-th term of the bias $b_{12l}(k)$ at $k$-th iteration.\\
\begin{equation*}
        \|b_{12}(k)\|
        \leq \frac{C_{L}p^2(k)\sqrt{m}}{6}\{[m^3-(m-1)^3]\alpha_l^2+(m-1)^3\alpha_l\alpha_0^3\}.
    \end{equation*}
Analogously, the $b_{22}(k)$, $b_1(k)$ and $b_2(k)$ are also bounded.\\

With the $\hat{\nabla}^2_{\theta_e\theta_e}L(\theta_l(k), \theta_e(k))$ and $\hat{\nabla}_{\theta_e}f(\theta_l(k),\theta_e(k))$, we estimate the $[\nabla^2_{\theta_e\theta_e}L(\theta_l(k), \theta_e(k))]^{-1}\nabla_{\theta_e}f(\theta_l(k),\theta_e(k))$ through conjugate gradient.\\
\begin{equation*}
        \|\hat{\nabla}^2_{\theta_e\theta_e}L(\theta_l(k), \theta_e(k))\|\\
        =\|\nabla^2_{\theta_e\theta_e}L(\theta_l(k), \theta_e(k))+b_{22}(k)\|
        \geq \|\lambda I + b_{22}(k)\|.\\
\end{equation*}
Then we can tune the parameters of  $b_{22}(k)$ such as $\Delta, \alpha_0$ and $\alpha_l$ to let $\|[\hat{\nabla}^2_{\theta_e\theta_e}L(\theta_l, \theta_e)]^{-1}\| \leq \frac{2}{\sqrt{m}\lambda}$ and make sure the $\hat{\nabla}^2_{\theta_e\theta_e}L(\theta_l, \theta_e) \geq \frac{\lambda}{2}I$ 
\begin{equation*}
    \begin{aligned}
        &\|b(k)\|,\\
        &=\|E[[\hat{\nabla}^2_{\theta_e\theta_e}L(\theta_l(k), \theta_e(k))]^{-1}\hat{\nabla}_{\theta_e}f(\theta_l(k),\theta_e(k))\\
        &- [\nabla^2_{\theta_e\theta_e}L(\theta_l(k), \theta_e(k))]^{-1}\nabla_{\theta_e}f(\theta_l(k),\theta_e(k))]\|,\\
        &\overset{(iv)}{\leq} \|E[\hat{\nabla}^2_{\theta_e\theta_e}L(\theta_l(k), \theta_e(k))]^{-1}(\hat{\nabla}^2_{\theta_e\theta_e}L(\theta_l(k), \theta_e(k)) 
        -\nabla^2_{\theta_e\theta_e}L(\theta_l(k), \theta_e(k)))\\
        &[\nabla^2_{\theta_e\theta_e}L(\theta_l(k), \theta_e(k))]^{-1}]\|
        \|\nabla_{\theta_e}f(\theta_l(k),\theta_e(k))\| 
        + \| \hat{\nabla}^2_{\theta_e\theta_e}L(\theta_l(k), \theta_e(k))]^{-1} b_2(k)\|,\\
        &\overset{(v)}{\leq}\|[\hat{\nabla}^2_{\theta_e\theta_e}L(\theta_l(k), \theta_e(k))]^{-1}\|\|\hat{\nabla}^2_{\theta_e\theta_e}L(\theta_l(k), \theta_e(k))
        -\nabla^2_{\theta_e\theta_e}L(\theta_l(k), \theta_e(k))\|\|\\
        &[\nabla^2_{\theta_e\theta_e}L(\theta_l(k), \theta_e(k))]^{-1}\|
        \|\nabla_{\theta_e}f(\theta_l(k),\theta_e(k))\| 
        + \|\hat{\nabla}^2_{\theta_e\theta_e}L(\theta_l(k), \theta_e(k))]^{-1}\|\|b_2{k}\|,\\
        &\overset{(vii)}{\leq} \frac{ 2C_{f_{\theta_e}}\|b_{22}(k)\|+2\sqrt{m}\lambda\|b_2(k)\|}{m\lambda^2},\\
        &\leq \frac{2C_{f_{\theta_e}}(C_{L}p^2(k)+C_{f}p^2(k))}{6\sqrt{m}\lambda^2}\{[m^3-(m-1)^3]\alpha_l^2+(m-1)^3\alpha_l\alpha_0^3\},\\
    \end{aligned}
\end{equation*}
where $(vii)$ uses the result of Lemma (\ref{upperlip}).

    \begin{equation*}
    \begin{aligned}
        &\|b_a(k)\| = E[\|\hat{\nabla} f(\theta_l(k), \theta_e(k)) \|] - \|\nabla f(\theta_l(k), \theta_e(k))\|,\\
        &\overset{(vii)}{=}E[\| \hat{\nabla}_{\theta_l} f(\theta_l(k),\theta_e(k))- \hat{\nabla}^2_{\theta_l\theta_e}L(\theta_l(k), \theta_e(k))[\hat{\nabla}^2_{\theta_e\theta_e}L(\theta_l(k), \theta_e(k))]^{-1}\\
        &\hat{\nabla}_{\theta_e}f(\theta_l(k),\theta_e(k))\|]
        - \|\nabla f(\theta_l(k), \theta_e(k))\|,\\
        &=\| \nabla_{\theta_l} f(\theta_l(k),\theta_e(k))+b_1(k)- (\nabla^2_{\theta_l\theta_e}L(\theta_l(k), \theta_e(k))+b_{12}(k))\\
        &([\nabla^2_{\theta_e\theta_e}L(\theta_l(k), \theta_e(k))]^{-1}
        \nabla_{\theta_e}f(\theta_l(k),\theta_e(k))+b(k))\|
        -\|\nabla f(\theta_l(k), \theta_e(k))\|\\
        &\overset{(v)}{\leq} \|b_1(k)\| + \|\nabla^2_{\theta_l\theta_e}L(\theta_l(k), \theta_e(k))\|\|b(k)\|\\
        &+\|[\nabla^2_{\theta_e\theta_e}L(\theta_l(k), \theta_e(k))]^{-1}\|\|\nabla_{\theta_e}f(\theta_l(k),\theta_e(k))\|\|b_{12}(k)\|
        + \|b_{12}(k)\|\|b(k)\|,\\
        &\overset{(vii)}{\leq} \frac{C_{f}p^2(k)\sqrt{n}}{6}\{[n^3-(n-1)^3]\alpha_l^2+(n-1)^3\alpha_l\alpha_0^3\}\\
        &+ \frac{2C_{L_{\theta_l\theta_e}}C_{L}p^2(k)+ 2C_{L_{\theta_l\theta_e}}\sqrt{m}\lambda C_{f}p^2(k)}{6\sqrt{m}\lambda^2}
        \{[m^3-(m-1)^3]\alpha_l^2+(m-1)^3\alpha_l\alpha_0^3\} \\
        &+\frac{{C_{f_{\theta_e}}}C_{L}p^2(k)}{6m\lambda^3}\{[m^3-(m-1)^3]\alpha_l^2+(m-1)^3\alpha_l\alpha_0^3\}\\
        &+ \frac{2p^4(k)C_{L} (C_{f_{\theta_e}}C_{L} + \sqrt{m}\lambda C_{f}) }{36\lambda^2}\{[m^3-(m-1)^3]\alpha_l^2
        +(m-1)^3\alpha_l\alpha_0^3\}^2,\\
    \end{aligned}
\end{equation*}
where the first $(vii)$ uses the result of Lemma (\ref{lemma2}), and the second $(vii)$ uses the result of Lemma (\ref{upperlip}).
\end{proof}
\subsection{ Proof of Theorem \ref{cr}}
Following the guidance of \cite{ghadimi2018approximation}, we choose an adaptive number of inner-loop iterations $t_k = \lceil \frac{\sqrt[4]{k+1}}{2}\rceil$. This value $\lceil \frac{\sqrt[4]{k+1}}{2}\rceil$ balances the trade-off between computational cost (minimizing inner loop iteration number) and approximation accuracy (reducing the error sufficiently fast), enabling efficient and stable convergence of the bi-level algorithm.
From the Lemma 2.2 of \cite{ghadimi2018approximation}, $\|\nabla f(\theta_l(k), \theta_e(k))-\nabla f(\theta_l(k), \theta_e^*(\theta_l(k)))\| \leq C\|\theta_e^*(\theta_l(k))-\theta_e(k)\|$, $\|\nabla f(\theta_l(k'), \theta_e^*(\theta_l(k')))-\nabla f(\theta_l(k), \theta_e^*(\theta_l(k)))\|\leq L_f\|\theta_l(k')-\theta_l(k)\|$ where $C = L_{f_{\theta_l}}+ \frac{L_{f_{\theta_e}}C_{L_{\theta_l\theta_e}}}{\lambda}+ C_{f_{\theta_e}}[\frac{L_{L_{\theta_l\theta_e}}}{\lambda}+\frac{L_{L_{\theta_e\theta_e}}C_{L_{\theta_l\theta_e}}}{\lambda^2}], L_f = \frac{(\bar{L}_{f_{\theta_e}}+C)C_{L_{\theta_l\theta_e}}}{\lambda}+ L_{f_{\theta_l}}+C_{f_{\theta_e}}[\frac{\bar{L}_{L_{\theta_l\theta_e}}C_{f_{\theta_e}}}{\lambda}+\frac{\bar{L}_{L_{\theta_e\theta_e}}C_{L_{\theta_l\theta_e}}}{\lambda^2}], Q_L = \frac{L_{L_{\theta_e}}}{\lambda}$ denotes the condition number of $L(\theta_l,\theta_e)$, $M =\max_{\theta_l\in \theta_l}\|\theta_e(0)-\theta_e^*(\theta_l)\|$, $D_{\theta_l} = \max_{x,y\in\theta_l}\{\|x-y\|\}$.\\
The proof of Theorem \ref{cr} is as follows:

\begin{proof}
First we compute the variance of $\hat{\nabla} f(\theta_l(k), \theta_e(k))$.\\
\begin{equation*}
    \begin{aligned}
        &\|\hat{\nabla} f(\theta_l(k), \theta_e(k)) \|,\\
        & \overset{(vii)}{=} \| \hat{\nabla}_{\theta_l} f(\theta_l(k),\theta_e(k))- \hat{\nabla}^2_{\theta_l\theta_e}L(\theta_l(k), \theta_e(k))[\hat{\nabla}^2_{\theta_e\theta_e}L(\theta_l(k), \theta_e(k))]^{-1}
        \hat{\nabla}_{\theta_e}f(\theta_l(k),\theta_e(k))\|,\\
        &=\| \nabla_{\theta_l} f(\theta_l(k),\theta_e(k))+b_1(k)- (\nabla^2_{\theta_l\theta_e}L(\theta_l(k), \theta_e(k))+b_{12}(k))\\
        &([\nabla^2_{\theta_e\theta_e}L(\theta_l(k), \theta_e(k))]^{-1}
        \nabla_{\theta_e}f(\theta_l(k),\theta_e(k))
        +b(k))\|,\\
        &\overset{(v)}{\leq} \|\nabla_{\theta_l} f(\theta_l(k),\theta_e(k))\| +\|b_1(k)\| +\|\nabla^2_{\theta_l\theta_e}L(\theta_l(k), \theta_e(k))\|\|[\nabla^2_{\theta_e\theta_e}L(\theta_l(k), \theta_e(k))]^{-1}\|\\
        &\|
        \nabla_{\theta_e}f(\theta_l(k),\theta_e(k))\|
        + \|\nabla^2_{\theta_l\theta_e}L(\theta_l(k), \theta_e(k))\|\|b(k)\| + \|b_1(k)\|\|[\nabla^2_{\theta_e\theta_e}L(\theta_l(k), \theta_e(k))]^{-1}\|\\
        &\|
        \nabla_{\theta_e}f(\theta_l(k),\theta_e(k))\| + \|b_1(k)\|\|b(k)\|,\\
        &\overset{(vii)}{\leq} C_{f_{\theta_l}} + \frac{C_{f}p^2(k)\sqrt{n}}{6}\{[n^3-(n-1)^3]\alpha_l^2+(n-1)^3\alpha_l\alpha_0^3\} \\&+\frac{C_{L_{\theta_l\theta_e}}C_{f_{\theta_e}}}{\sqrt{m}\lambda}+ \frac{2C_{L_{\theta_l\theta_e}}C_{f_{\theta_e}}(C_{L}p^2(k)+C_{f}p^2(k))}{6\sqrt{m}\lambda^2}\{[m^3-(m-1)^3]\alpha_l^2+(m-1)^3\alpha_l\alpha_0^3\}\\
        &+
        \frac{C_{f_{\theta_e}}}{\sqrt{m}\lambda}\frac{C_{f}p^2(k)\sqrt{n}}{6}\{[n^3-(n-1)^3]\alpha_l^2+(n-1)^3\alpha_l\alpha_0^3\} \\
        &+ \frac{2C_{f}p^2(k)\sqrt{n}C_{f_{\theta_e}}(C_{L}p^2(k)+C_{f}p^2(k))}
        {36\sqrt{m}\lambda^2}\\
        &\{[n^3-(n-1)^3]\alpha_l^2+(n-1)^3\alpha_l\alpha_0^3\}\{[m^3-(m-1)^3]\alpha_l^2+(m-1)^3\alpha_l\alpha_0^3\},
    \end{aligned}
\end{equation*}
where the first $(vii)$ uses the result of Lemma (\ref{lemma2}), and the second $(vii)$ uses the result of Lemma (\ref{upperlip}) and Lemma (\ref{bias}).\\
Since $\hat{\nabla} f(\theta_l(k), \theta_e(k))$ is bounded,
\begin{equation*}
    \begin{aligned}
        &Var(\hat{\nabla} f(\theta_l(k), \theta_e(k)) ),\\
        &\leq (C_{f_{\theta_l}} + \frac{C_{f}p^2(k)\sqrt{n}}{6}\{[n^3-(n-1)^3]\alpha_l^2+(n-1)^3\alpha_l\alpha_0^3\} \\&+\frac{C_{L_{\theta_l\theta_e}}C_{f_{\theta_e}}}{\sqrt{m}\lambda}+ \frac{2C_{L_{\theta_l\theta_e}}C_{f_{\theta_e}}(C_{L}p^2(k)+C_{f}p^2(k))}{6\sqrt{m}\lambda^2}\{[m^3-(m-1)^3]\alpha_l^2+(m-1)^3\alpha_l\alpha_0^3\}\\
        &+
        \frac{C_{f_{\theta_e}}}{\sqrt{m}\lambda}\frac{C_{f}p^2(k)\sqrt{n}}{6}\{[n^3-(n-1)^3]\alpha_l^2+(n-1)^3\alpha_l\alpha_0^3\} \\
        &+ \frac{2C_{f}p^2(k)\sqrt{n}C_{f_{\theta_e}}(C_{L}p^2(k)+C_{f}p^2(k))}
        {36\sqrt{m}\lambda^2}\\
        &\{[n^3-(n-1)^3]\alpha_l^2+(n-1)^3\alpha_l\alpha_0^3\}\{[m^3-(m-1)^3]\alpha_l^2+(m-1)^3\alpha_l\alpha_0^3\})^2.
    \end{aligned}
\end{equation*}
Then we need to find the total bias, the total bias $b_t(k)$ is the sum of the bias from approximation and $\nabla f(\theta_l(k), \theta_e(k))-\nabla f(\theta_l(k), \theta_e^*(\theta_l(k)))$\\
\begin{equation*}
    \begin{aligned}
        &\|b_t(k)\|,\\
        &=\|E[\nabla f(\theta_l(k), \theta_e(k))] - \nabla f(\theta_l(k), \theta_e^*(k))\|,\\
        &\overset{(vii)}{\leq}\|b_a(k)\|+ \|r_{li}(\frac{Q_L-1}{Q_L+1})^{t_k} \|\theta_e(0)-\theta_e^*(\theta_l(k))\|,\\
    \end{aligned}
\end{equation*}
where $(vii)$ adds the approximation error $r_{li}(\frac{Q_L-1}{Q_L+1})^{t_k} \|\theta_e(0)-\theta_e^*(\theta_l(k))$ from the lower level. 
Next step is to find the bound for $E[\|\nabla f(\theta_l(k),\theta_e^*(\theta_l(k)))\|^2]$.\\
\begin{equation*}
    \begin{aligned}
        &f(\theta_l(k+1),\theta_e^*(\theta_l(k+1))),\\
        &\overset{(vi)}{\leq} f(\theta_l(k),\theta_e^*(\theta_l(k)))
        + \langle \nabla f(\theta_l(k),\theta_e^*(\theta_l(k))), \theta_l(k+1) -\theta_l(k)\rangle 
        + \frac{L_f}{2} \| \theta_l(k+1) -\theta_l(k)\|^2,\\  
        &=f(\theta_l(k),\theta_e^*(\theta_l(k)))
        -\alpha_k\langle \nabla f(\theta_l(k),\theta_e^*(\theta_l(k))), \hat{\nabla}f(\theta_l(k),\theta_e(k))\rangle
        +\frac{L_f\alpha_k^2}{2} \|\hat{\nabla}f(\theta_l(k),\theta_e(k))\|^2.
    \end{aligned}
\end{equation*}

The expectation of $f(\theta_l(k+1),\theta_e^*(\theta_l(k+1)))$ becomes:
\begin{equation*}
    \begin{aligned}
        &E[f(\theta_l(k+1),\theta_e^*(\theta_l(k+1)))],\\
        &\leq  f(\theta_l(k),\theta_e^*(\theta_l(k)))
        -\alpha_k\langle \nabla f(\theta_l(k),\theta_e^*(\theta_l(k))), \nabla f(\theta_l(k),\theta_e^*(\theta_l(k))) + b_t(k)\rangle\\
        &+\frac{L_f\alpha_k^2}{2} E\|\nabla f(\theta_l(k),\theta_e^*(\theta_l(k)))+\hat{\nabla}f(\theta_l(k),\theta_e(k))
        -\nabla f(\theta_l(k),\theta_e^*(\theta_l(k)))\|^2,\\
        &\leq f(\theta_l(k),\theta_e^*(\theta_l(k)))
        -\alpha_k\langle \nabla f(\theta_l(k),\theta_e^*(\theta_l(k))), \nabla 
        f(\theta_l(k),\theta_e^*(\theta_l(k))) + b_t(k)\rangle\\
        &+\frac{L_f\alpha_k^2}{2} Var(\hat{\nabla}f(\theta_l(k),\theta_e(k))) 
        + \frac{L_f\alpha_k^2}{2} E\|\nabla f(\theta_l(k),\theta_e^*(\theta_l(k)))\|^2 \\
        &+  L_f\alpha_k^2\langle \nabla f(\theta_l(k),\theta_e^*(\theta_l(k))), b_t(k)\rangle,\\
        &=f(\theta_l(k),\theta_e^*(\theta_l(k)))
        -(\alpha_k - \frac{L_f\alpha_k^2}{2}) \|\nabla f(\theta_l(k),\theta_e^*(\theta_l(k)))\|^2\\
        &-(\alpha_k - L_f\alpha_k^2) \langle \nabla f(\theta_l(k),\theta_e^*(\theta_l(k))), b_t(k)\rangle
        +\frac{L_f\alpha_k^2}{2} Var(\hat{\nabla}f(\theta_l(k),\theta_e(k))) + \frac{L_f\alpha_k^2}{2} \|b_t(k)\|^2.\\
    \end{aligned}
\end{equation*}
Choose $\alpha_k \leq \frac{1}{L_f}$ and with the fact $2\langle \nabla f(\theta_l(k),\theta_e^*(\theta_l(k))), b_t(k)\rangle  \leq \|\nabla f(\theta_l(k),\theta_e^*(\theta_l(k)))\|^2 + \|b_t(k)\|^2$.
\begin{equation*}
    \begin{aligned}
        &E[f(\theta_l(k+1),\theta_e^*(\theta_l(k+1)))],\\
        &\leq f(\theta_l(k),\theta_e^*(\theta_l(k)))-\frac{\alpha_k}{2}\|\nabla f(\theta_l(k),\theta_e^*(\theta_l(k)))\|^2\\  
        & + \frac{\alpha_k}{2} \|b_t(k)\|^2+\frac{L_f\alpha_k^2}{2} Var(\hat{\nabla}f(\theta_l(k),\theta_e(k))),\\
    \end{aligned}
\end{equation*}
Rearrange terms,\\
\begin{equation*}
    \begin{aligned}
        &\sum_{k=0}^{K-1} \frac{\alpha_k}{2}E[\|\nabla f(\theta_l(k),\theta_e^*(\theta_l(k)))\|^2],\\
        &\leq f(\theta_l(0),\theta_e^*(\theta_l(0))) -f^*
        + \sum_{k=0}^{K-1} (\frac{\alpha_k}{2}\|b_t(k)\|^2 + \frac{L_f\alpha_k^2}{2} Var(\hat{\nabla}f(\theta_l(k),\theta_e(k)))).\\
    \end{aligned}
\end{equation*}
 For $\|b_t(k)\|^2$, it is the linear combination of $p^4(k),p^6(k),p^8(k),p^2(k)(\frac{Q_L-1}{Q_L+1})^{t_k} + ,p^4(k)(\frac{Q_L-1}{Q_L+1})^{t_k}, (\frac{Q_L-1}{Q_L+1})^{2t_k} $. For $Var(\hat{\nabla}f(\theta_l(k),\theta_e(k))))$, it is the linear combination of $1,p^2(k),p^4(k),p^6(k),p^8(k) $. For simplification, we use $C_{si}>0, i=1,2,\dots$ to represent the constant for all combinations of terms involve $p(k)$, $\alpha_k $ and $t_k $. Then we can continue the calculation as follows: \\
\begin{equation*}
    \begin{aligned}
        &\frac{1}{K}\sum_{k=0}^{K-1}E[\|\nabla f(\theta_l(k),\theta_e^*(\theta_l(k)))\|^2],\\
        &\leq \frac{2}{K\alpha_k}f(\theta_l(0),\theta_e^*(\theta_l(0))) -f^*
        + \frac{1}{K}\sum_{k=0}^{K-1} (\|b_t(k)\|^2 + L_f\alpha_k Var(\hat{\nabla}f(\theta_l(k),\theta_e(k)))),\\
        &\overset{(viii)}{\leq} \frac{ 2(f(\theta_l(0),\theta_e^*(\theta_l(0))) -f^*)}{\alpha_k K}
        +\frac{C_{s1}}{K}\sum_{k=0}^{K-1}p^4(k) + \frac{C_{s2}}{K}\sum_{k=0}^{K-1} p^6(k) +\frac{C_{s3}}{K}\sum_{k=0}^{K-1} p^8(k) \\
        &+\frac{C_{s4}}{K}\sum_{k=0}^{K-1} p^2(k)(\frac{Q_L-1}{Q_L+1})^{t_k} 
        + \frac{C_{s5}}{K}\sum_{k=0}^{K-1} p^4(k)(\frac{Q_L-1}{Q_L+1})^{t_k} 
        + \frac{C_{s6}}{K}\sum_{k=0}^{K-1}(\frac{Q_L-1}{Q_L+1})^{2t_k} \\
        &+ \frac{C_{s7}}{K}\sum_{k=0}^{K-1} \alpha_k + \frac{C_{s8}}{K}\sum_{k=0}^{K-1} \alpha_k p^2(k)+\frac{C_{s9}}{K}\sum_{k=0}^{K-1} \alpha_k p^4(k)
        + \frac{C_{s10}}{K}\sum_{k=0}^{K-1} \alpha_k p^6(k) \\
        &+ \frac{C_{s11}}{K}\sum_{k=0}^{K-1} \alpha_k p^8(k),
    \end{aligned}
\end{equation*}
where $(viii)$ expands each term of $\|b_t(k)\|^2$ and $Var(\hat{\nabla}f(\theta_l(k),\theta_e(k))))$.
Choose $p(k) = \frac{1}{k}$, $\alpha_k = \frac{1}{L_f \sqrt{K}}$, $t_k = \lceil \frac{\sqrt[4]{k+1}}{2}\rceil$. Since $0\leq\frac{Q_L-1}{Q_L+1} < 1$, we can conclude $\sum_{k=0}^{K-1}p^2(k)(\frac{Q_L-1}{Q_L+1})^{t_k} < \sum_{k=0}^{K-1}p^2(k)$ when $\frac{Q_L-1}{Q_L+1} \neq 0$, then the convergence rate is as follows:
\begin{equation*}
        \frac{1}{K} \sum_{k=0}^{K-1}E[\|\nabla f(\theta_l(k),\theta_e^*(\theta_l(k)))\|^2]
        \leq  \frac{C_{s12}}{\sqrt{K}} +  \frac{C_{s13}}{K} + \frac{C_{s14}}{K\sqrt{K}}.
\end{equation*}
As $K\to \infty$, $\frac{1}{K}\sum_{k=0}^{K-1}E[\|\nabla f(\theta_l(k),\theta_e^*(\theta_l(k)))\|^2] \to 0$, which shows that $E[\nabla f(\theta_l(k),\theta_e^*(\theta_l(k)))]$ decreases at the rate of $\mathcal{O}(\frac{1}{\sqrt{K}}+\frac{1}{K}+\frac{1}{K\sqrt{K}})$. 
\end{proof}

\subsection{Proof of Corollary \ref{policycon}}
Define the cumulative reward function of the expert as $J_{e}(\pi) \triangleq E^{\pi}[\sum_{h=0}^{H-1} \gamma^h r_{\theta_e}(s^h,a^h)]$. If $r_{\theta_e}$ is a linear reward function, we have $r_{\theta_e} \triangleq \langle\theta_e, \phi(s,a)\rangle$ where the feature $\phi(s,a) \in \mathbb{R}^{d_{\theta_e}}$ and $d_{\theta_e} $ is the dimension of $\theta_e$. Then the expert's feature expectation is formulated as $\mu_f(\pi) \triangleq E^{\pi}[\sum_{h=0}^{H-1} \gamma^h \phi(s^h,a^h)]$. From Theorem \ref{cr}, we can get $\|\mu_f(\pi_{\theta_e})-\mu_f(\pi_e)\|^2$ decreases in $\mathcal{O}(\frac{1}{\sqrt{K}})$.  
\begin{proof}
\begin{equation*}
    \begin{aligned}
            J_e(\pi_{\theta_e})- J_e(\pi_e) &= \langle\theta_e,\mu_f(\pi_{\theta_e})-\mu_f(\pi_e)\rangle\\
            &\leq
            \max_{\theta_e \in \Theta} \langle\theta_e,\mu_f(\pi_{\theta_e})-\mu_f(\pi_e)\rangle \\
            &\overset{(v)}\leq \max_{\theta_e \in \Theta} \|\theta_e\|\|\mu_f(\pi_{\theta_e})-\mu_f(\pi_e)\|,\\
            &= \|\mu_f(\pi_{\theta_e})-\mu_f(\pi_e)\|,\\
            &\overset{(vii)}\leq \frac{C_{15}}{\sqrt[4]{K}},
    \end{aligned}
\end{equation*}
where $(vii)$ uses the result $\|\mu_f(\pi_{\theta_e})-\mu_f(\pi_e)\|^2$ decreases in $\mathcal{O}(\frac{1}{\sqrt{K}})$ and $C_{15}$ is the constant number which includes all influence factors other than $K$.
\end{proof}

\section{Proof of Theorem \ref{cc}}\label{t2}
The prove of Theorem \ref{cc} is as follows:
According to the Algorithm 1 in \cite{ziebart2008maximum}, we need to compute the state frequency for the $\mu_e(\pi_{\theta_l, \theta_e})$. For each state-action pair, it needs to recursively compute for up to $H$ iterations. As there are $H$ state-action pairs in one demonstration, the computational complexity for the $\mu_e(\pi_{\theta_l, \theta_e})$ is $\mathcal{O}(H^2)$ as we see $H$ as the deterministic factor. Analogously, the computational complexity for $\mu_l(\pi_{\theta_l, \theta_e}), \mu_e(s,a),\mu_e(s),\mu_l(s,a),\mu_l(s),J_{l}(\pi_{\theta_l, \theta_e})$ are $\mathcal{O}(H^2)$. We can see that these factors with computational complexity $\mathcal{O}(H^2)$ are expected to be calculated through backpropagation, and the backpropagation leads to $\mathcal{O}(H^2)$. However, in coding, we can calculate these factors by sampling a finite number of trajectories, and the computational complexity becomes $\mathcal{O}(H)$.  \\

If we directly compute $\nabla_{\theta_l} f(\theta_l(k),\theta_e(k))$, $\nabla^2_{\theta_l\theta_e}L(\theta_l(k), \theta_e(k))$, $\nabla^2_{\theta_e}L(\theta_l(k), \theta_e(k))$, and $\nabla_{\theta_e}f(\theta_l(k),\theta_e(k))$ instead of approximating, use the expression of $\nabla_{\theta_l} f(\theta_l,\theta_e) = E^{\pi_{\theta_l, \theta_e}}[\sum^{H-1}_{h=0} \gamma^h[(\mu_l(s^h,a^h) - \mu_l(s^h))(J_{l}(s^h,a^h))^T ]]+ \mu_l(\pi_{\theta_l, \theta_e})$ as an example, $\mu_l(s)$ is inside an expectation from $h=0$ to $H-1$, we need to sum up  $\mu_l(s)$ for $H$ times. As a result, the computational complexity for $\nabla_{\theta_l} f(\theta_l,\theta_e)$ is $\mathcal{O}(H^2)$. Analogously, $\nabla_{\theta_e} f(\theta_l,\theta_e)$, $\nabla^2_{\theta_l\theta_e}L(\theta_l, \theta_e)$, $\nabla^2_{\theta_e\theta_e}L(\theta_l, \theta_e)$ are all same.\\

For SPSA, the required terms are $f(\theta_l,\theta_e), \nabla_{\theta_e}L(\theta_l,\theta_e),\nabla_{\theta_l}L(\theta_l,\theta_e)$, these terms are summing $H$ bounded values, so their computational complexities are $\mathcal{O}(H)$. As SPSA utilizes $f(\theta_l,\theta_e), \nabla_{\theta_e}L(\theta_l,\theta_e),\nabla_{\theta_l}L(\theta_l,\theta_e)$ to approximate $\nabla_{\theta_l} f(\theta_l,\theta_e),\nabla_{\theta_e} f(\theta_l,\theta_e)$, $\nabla^2_{\theta_l\theta_e}L(\theta_l, \theta_e)$, $\nabla^2_{\theta_e\theta_e}L(\theta_l, \theta_e)$ instead of directly computing. The overall computational complexity of using SPSA matches the computational complexity of  $f(\theta_l,\theta_e), \nabla_{\theta_e}L(\theta_l,\theta_e),\nabla_{\theta_l}L(\theta_l,\theta_e)$. Therefore, the overall computational of SPSA is $\mathcal{O}(H)$. \\

According to the algorithm of SPSA, more policies need to be found compared to directly compute the hypergradient. Use soft q learning \cite{haarnoja2017reinforcement} as an example, RL and MARL algorithms can be considered as sampling a finite number trajectories for each epoch. Therefore the computational cost for each epoch is $\mathcal{O}(H)$ and the overall computational cost is $\mathcal{O}(eH)$ where $e$ is the total number of epochs. As we consider $H$ as the decision factor, the computational cost of the multi-agent RL is $\mathcal{O}(H)$, which matches the computational complexity of SPSA. 

As a result, the computational complexity of SPSA is dominated by $\mathcal{O}(H)$, and directly calculating the hypergradient is dominated by $\mathcal{O}(H^2)$.

\section{Experiment Detail}\label{edetail}
The details of the experiments are shown in this section. All Python3 codes are run on a Windows 10 desktop with 13th Gen Intel(R) Core(TM) i7-13700KF CPU and 32 GB of RAM. For each combination of algorithms and environments, we run 10 times to calculate mean values and standard deviations at each iteration. Then the calculated mean values and standard deviations are plotted as shown in Section \ref{experiments} figures. 

\subsection{MPE}\label{mpesetup}
The state, action, and observation spaces for the adversary and good agents are continuous. For the adversary, it can observe the relative distance to the landmarks and the good agents, therefore the observation of the adversary is $o_a = [p_{l1}- p_{a},p_{l2}- p_{a},p_{g1} - p_{a}, p_{g2} - p_{a}]$ where $p_{l1}$ is the position of the first landmark, $p_{l2}$ is the position of the second landmark, $p_{g1}$ is the position of the first good agent, $p_{g2}$ is the position of the second good agent. For each good agent, it can observe the relative distance to the target landmark, the landmarks, the adversary, and another good agent, therefore the observations for two good agents are $o_{g1} = [p_{tl} - p_{g1},p_{l1}- p_{g1},p_{l2}- p_{g1},p_{a} - p_{g1}, p_{g2} - p_{g1} ]$ and  $o_{g2} = [p_{tl} - p_{g2},p_{l1}- p_{g2},p_{l2}- p_{g2},p_{a} - p_{g2}, p_{g1} - p_{g2} ]$ where $p_{tl}$ is the position of the target landmark. The actions of the adversary and the good agents are the velocities between $0$ and $1$ in four directions (left, right, down, up). Two good agents share the same return, which is rewarded based on the minimum distance of any agent to the target landmark and is penalized based on the distance between the adversary and the target landmark, therefore the reward of good agents is $r_g = -\min(\|p_{tl} - p_{g1}\|_2,\|p_{tl} - p_{g2}\|_2) + \|p_{tl} - p_{a}\|_2$. The reward of the adversary is based on the distance to the target of the adversary, therefore $r_a = -\|p_{ta} - p_{a}\|_2$, where $p_{ta}$ is the position of the adversary's target. In our simulation, we consider observations as states of the MG. 
\subsection{SMAC}\label{smacexp}
In this scenario, we can only access the state, observation, and action information of agents. The enemy is controlled by a built-in game AI. We consider observations as the states of the agents during learning. For each agent, it can observe the following information corresponding to another agent and enemy: relative distance, relative x, relative y, health, shield, and unit type. If an agent or an enemy is under attack, the shield reduces first and health reduces after the shield disappears. There are $7$ possible actions for each agent: move north, move south, move east, move west, attack the enemy, stop, and no-op. At each time step, each agent can know which action among these $7$ possible actions are available and the agent chooses one action from available actions. For example, the attack action is available when the enemy is in the shooting range of the agent. The agent can only choose no-op when its own health is $0$. The agent gains rewards based on the damage dealt to the enemy and if the enemy is defeated. 
\subsection{Human-robot interaction}\label{hriexp}
The state of the robot is its location $s_r = (x_r,y_r) \in \mathbb{R}^2$, the action of the robot includes the horizontal and vertical velocities and defined as $a_r = (v_{rx}, v_{ry}), $ where $ v_{rx} \in [-0.1, 0.1], v_{ry} \in [-0.1, 0.1]$. Similarly, the state of the human is $s_h = (x_h,y_h) \in \mathbb{R}^2 $, the action of the human is $a_h = (v_{hx}, v_{hy}), $ where $ v_{hx} \in [-0.1, 0.1], v_{hy} \in [-0.1, 0.1]$. At each time step $t$, the 
robot chooses the action $a_r(t)$ based on the current joint state $(s_r(t),s_h(t))$ and moves to the next state $x_r(t+1) =x_r(t) + v_{rx}(t), y_r(t+1) = y_r(t) + v_{ry}(t)$. Analogously, the motion dynamics of the human is given by $x_h(t+1) =x_h(t) + v_{hx}(t), y_h(t+1) = y_h(t) + v_{hy}(t)$.  In the experiment, the robot starts from an initial state $s_r(0) \in [-0.25,0.25] \times [-0.4,0.5]$  and aims to reach a circle goal region whose center is at $(0,0.5)$ with radius $0.05$. The human starts from $s_h(0) \in [0.4,0.5] \times [-0.25,0.25]$ and aims to reach a circle goal region whose center is at $(-0.5,0)$ with radius $0.05$. Both the robot and the human are penalized when a collision happens. 
\begin{figure*}[htb]
\centering
\vspace{-4mm}
\subfigure[Robot's cumulative reward]{ 
\label{lr}
\includegraphics[width=4.25cm]{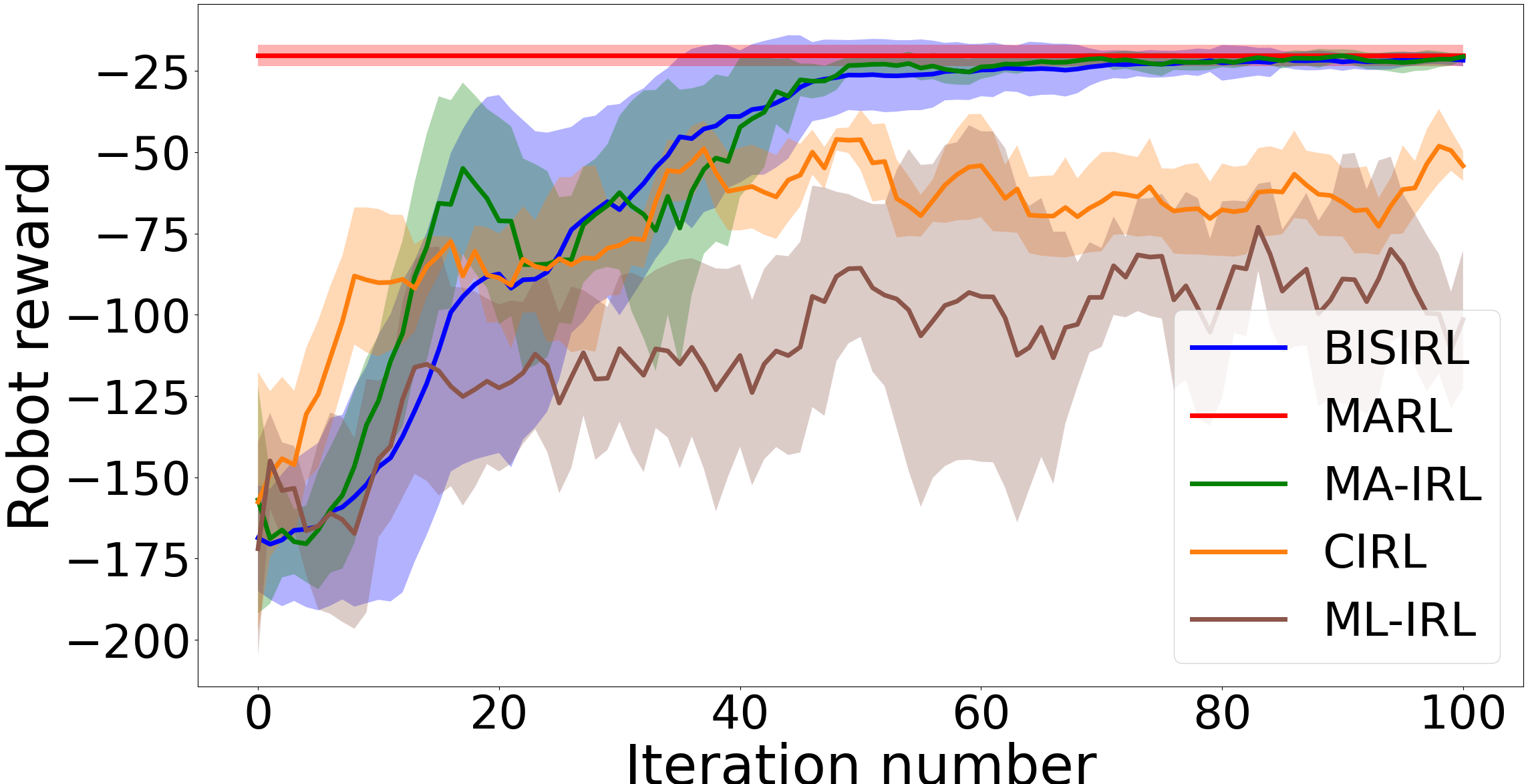}

}
\hspace{-1mm}
\subfigure[Human's cumulative reward]{
\label{er}
\includegraphics[width=4.25cm]{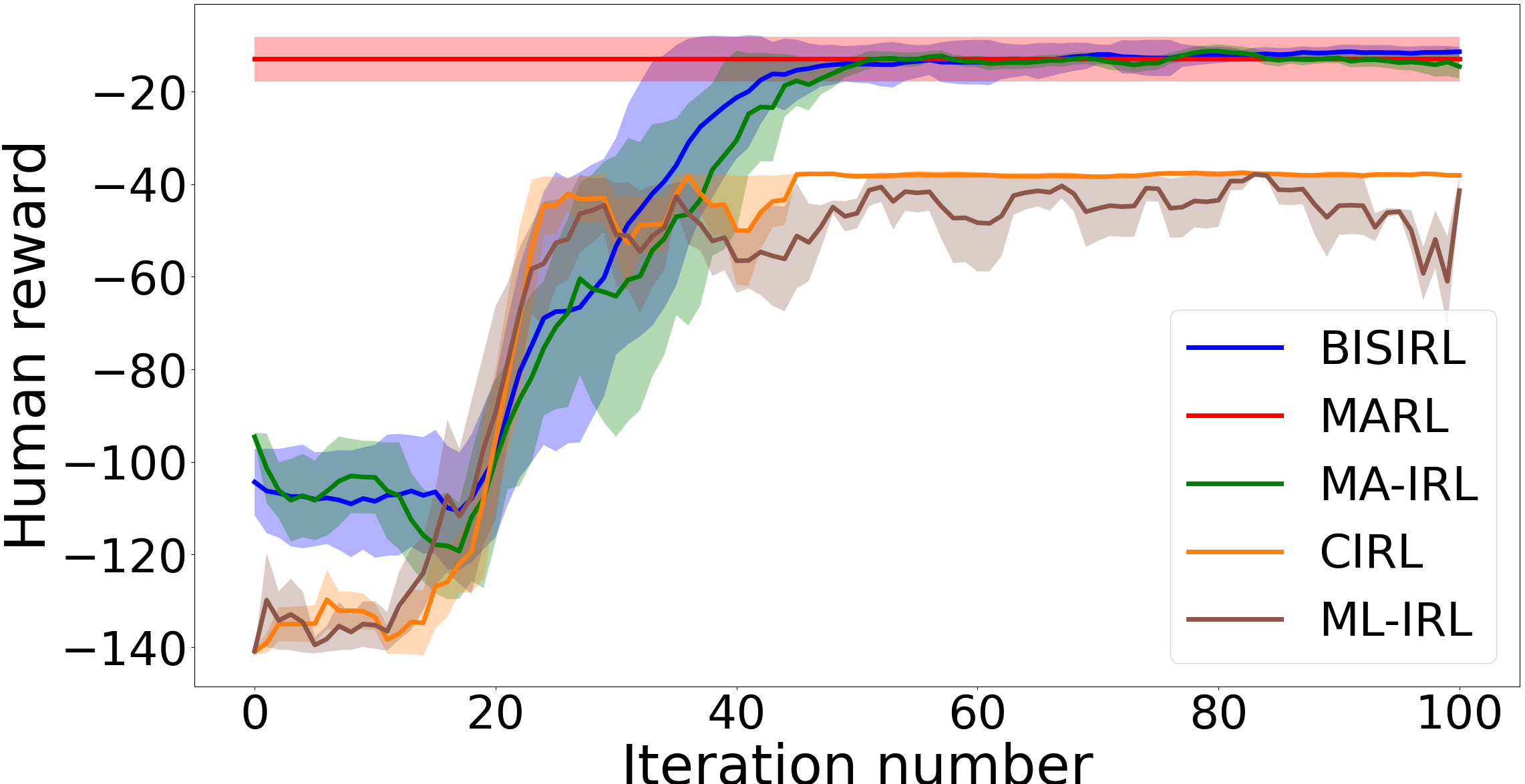}
}
\subfigure[Collision rate]{
\label{collision}
\includegraphics[width=4.25cm]{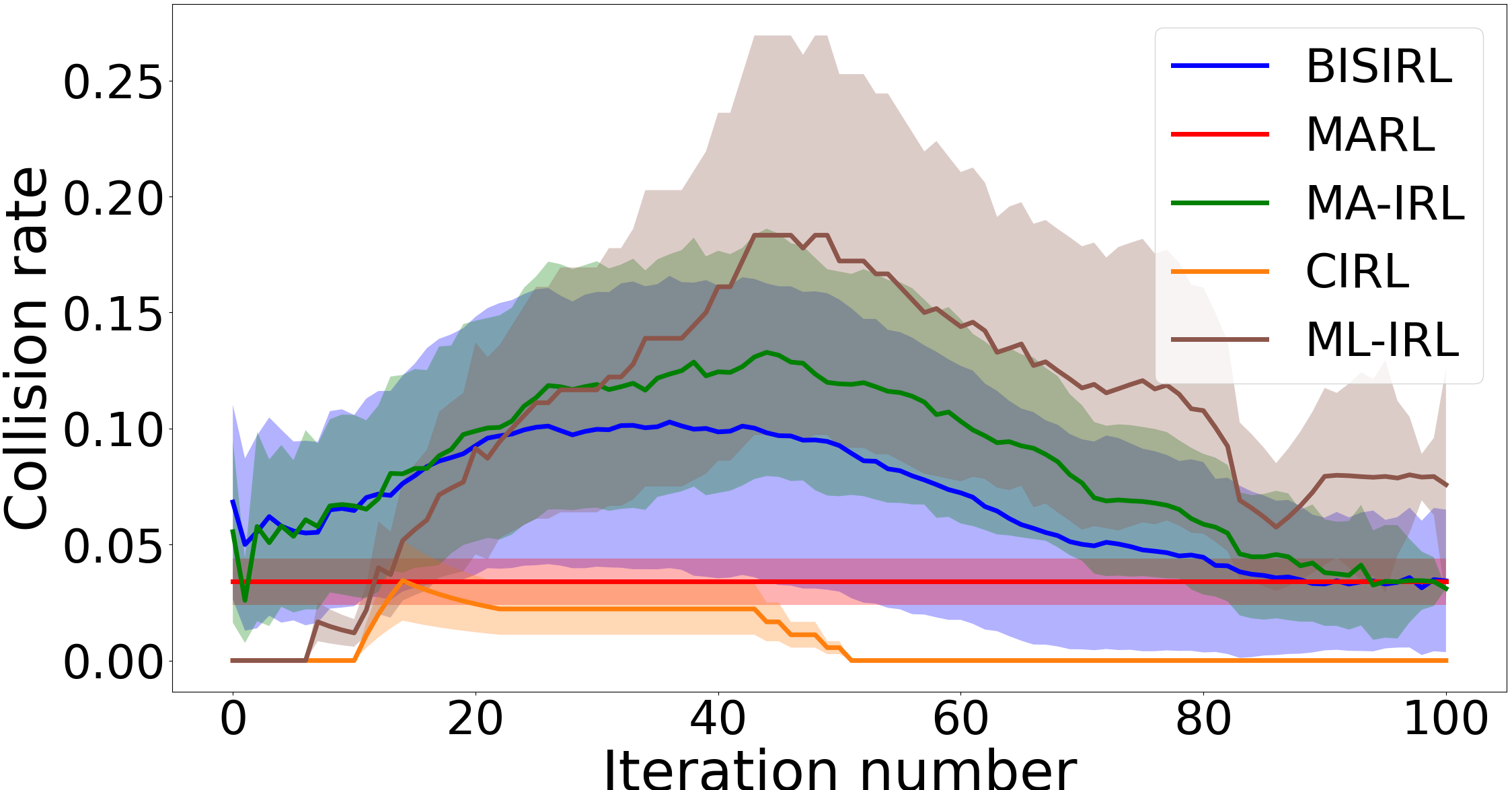}
}
\vspace{-2mm}
\hspace{-1mm}
\subfigure[Goal reaching rate]{
\label{goal}
\includegraphics[width=4.25cm]{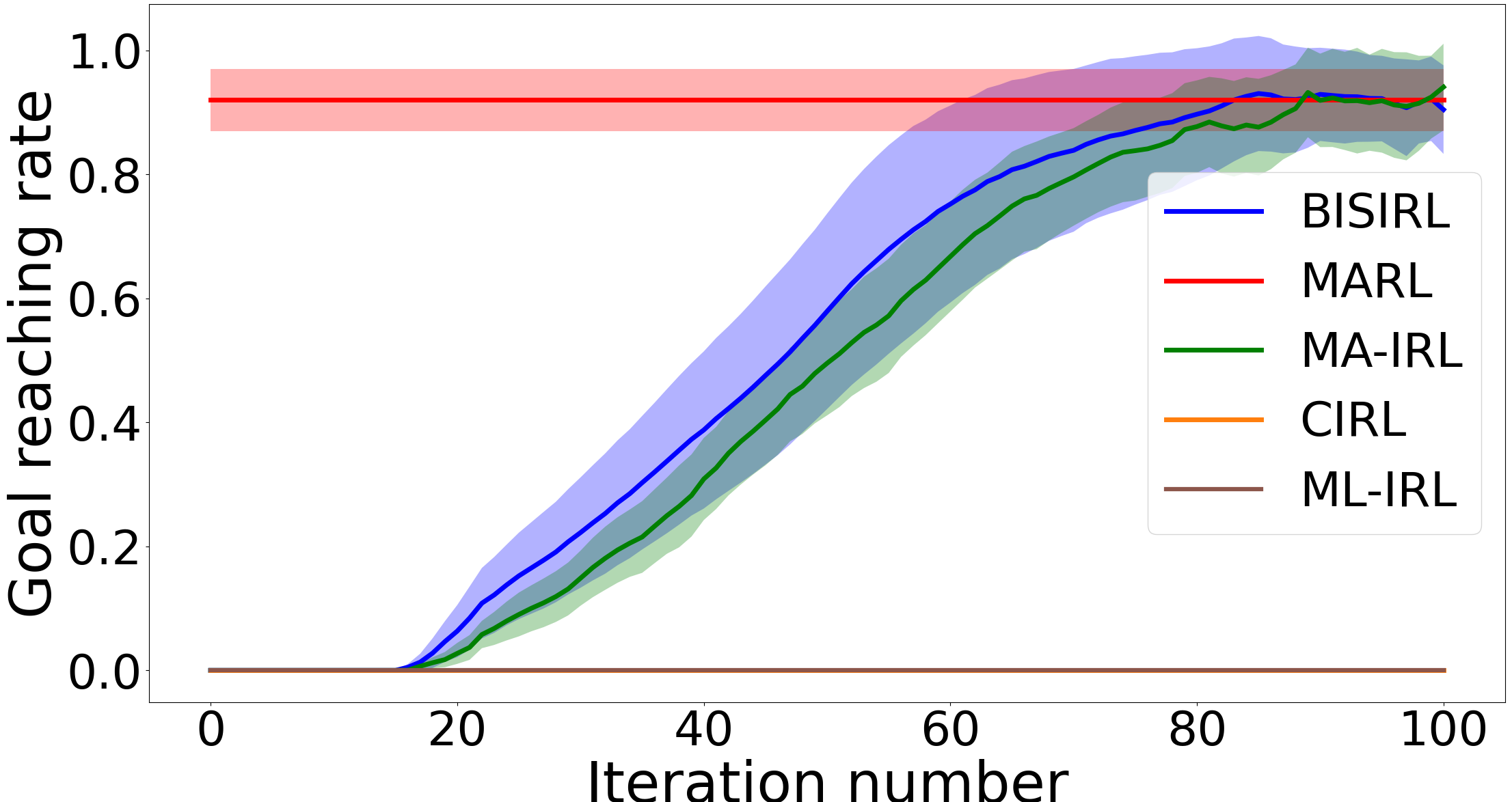}
}
\subfigure[Sample trajectory]{
\label{traj}
\includegraphics[width=4.25cm]{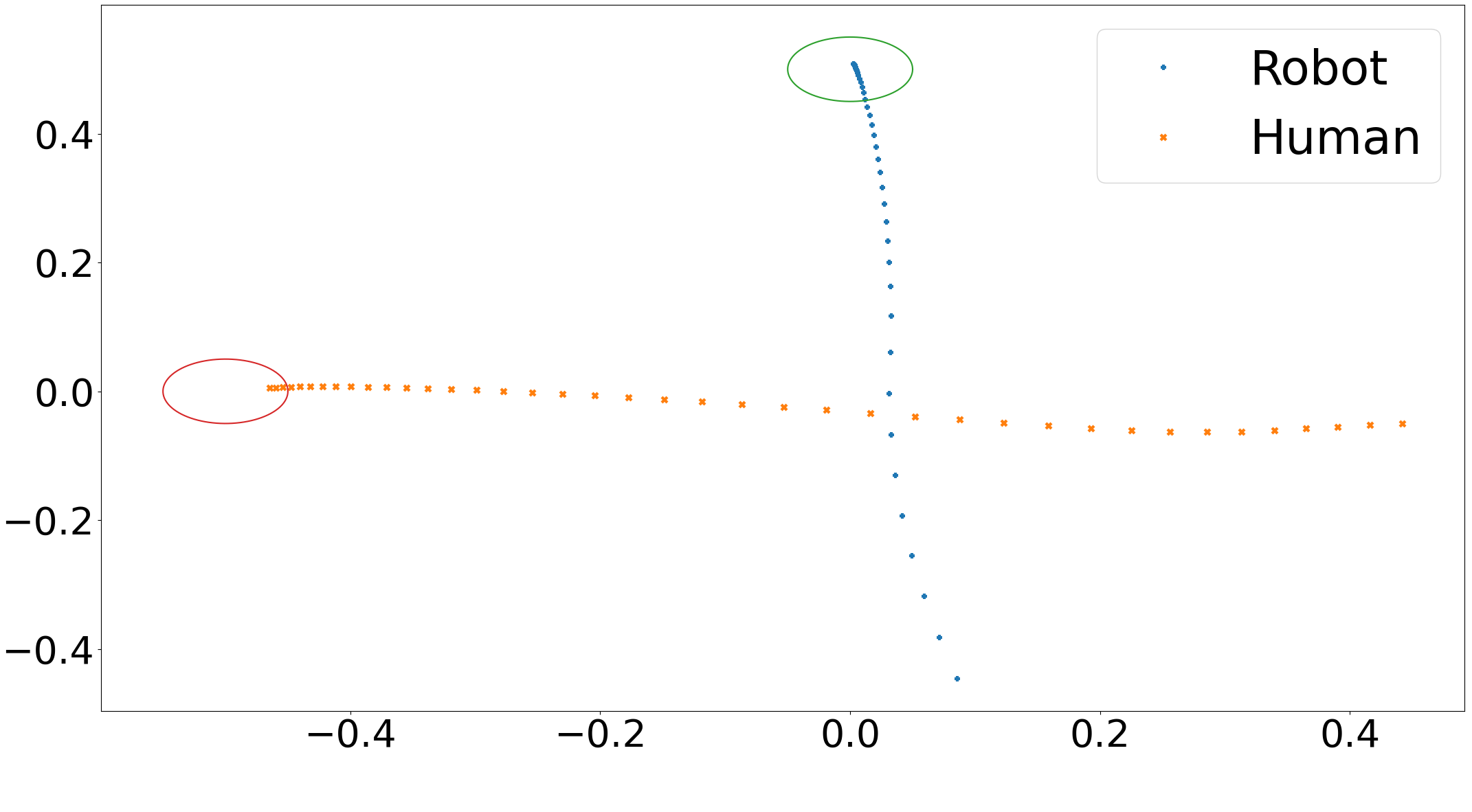}
}
\subfigure[Distance]{
\label{dis}
\includegraphics[width=4.25cm]{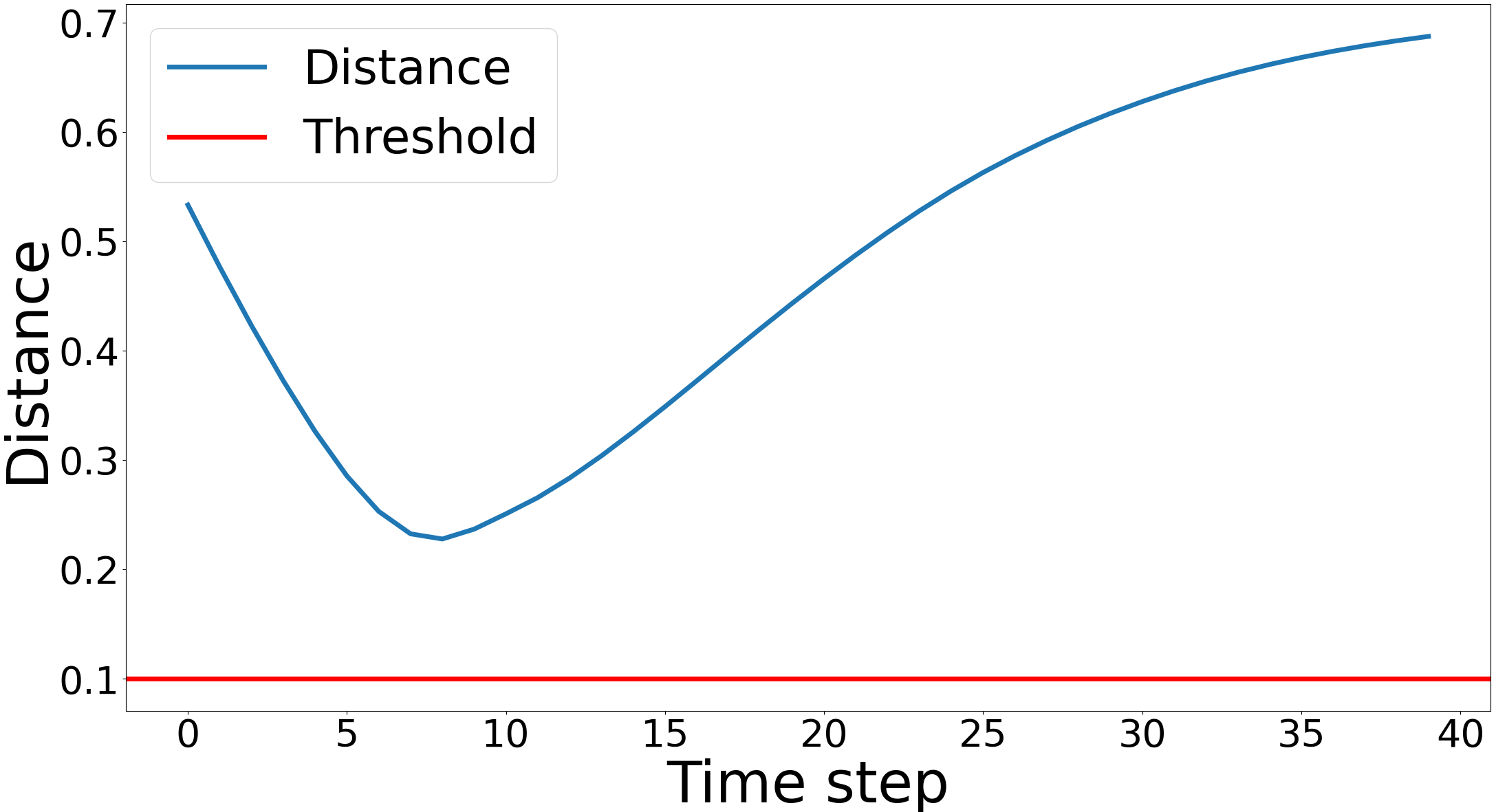}
}
\vspace{-1mm}
\caption{Human–robot interaction simulation results. The cumulative rewards are calculated as described in Figure \ref{mpeexp}. Similarly, the collision and goal-reaching rates are computed based on the policies corresponding to the learned reward functions. A collision is defined to occur when the distance between the human and the robot falls below a threshold. The goal is considered reached only when both the human and the robot reach their respective destinations.}
\label{hri}
\end{figure*}
This experiment examines a HRI scenario at a crosswalk, where both the robot (learner) and the human (expert) aim to reach their respective destinations as efficiently as possible. Additionally, the robot is expected to respect the human’s right-of-way. Both agents operate in continuous state and action spaces: each agent has a $4$-dimensional state and a $2$-dimensional action space.

Figures \ref{lr} and \ref{er} demonstrate that the cumulative rewards of both the robot and the human converge to their respective ground-truth values under BISIRL and MA-IRL, consistent with earlier results. Figure \ref{collision} shows that the collision rates for BISIRL and MA-IRL also converge to the ground-truth value, indicating that the robot successfully learns to respect the human’s right-of-way. Initially, with randomly initialized reward functions, both agents move unpredictably and rarely collide. As the learned reward functions improve, both agents pursue goals more directly, increasing collisions temporarily, followed by a decline as collision-avoidance is prioritized. Figure \ref{goal} further shows that the goal-reaching rates under BISIRL and MA-IRL similarly converge to the ground-truth value, confirming that the learned reward functions enable both agents to reliably reach their destinations. After approximately 100 iterations, reward functions learned by BISIRL and MA-IRL result in policies that achieve high goal-reaching success and low collision frequency. In contrast, CIRL and ML-IRL fail to match the ground-truth performance across all three metrics, highlighting their limitations in capturing the objectives of both agents. Figure \ref{traj} presents a representative trajectory under the learned reward functions of BISIRL, while Figure \ref{dis} plots the inter-agent distance along this trajectory, confirming that both agents reach their goals while maintaining safe separation above the required threshold.

\subsection{Security}\label{csexp}
\begin{figure*}[htbp]
\centering

\subfigure{ 
\label{fig:ag}
\includegraphics[width=4.25cm]{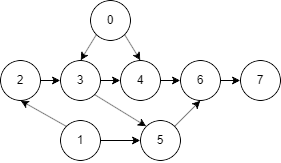}

}
\hspace{-1mm}
\subfigure{ 
\label{dr}
\includegraphics[width=4.25cm]{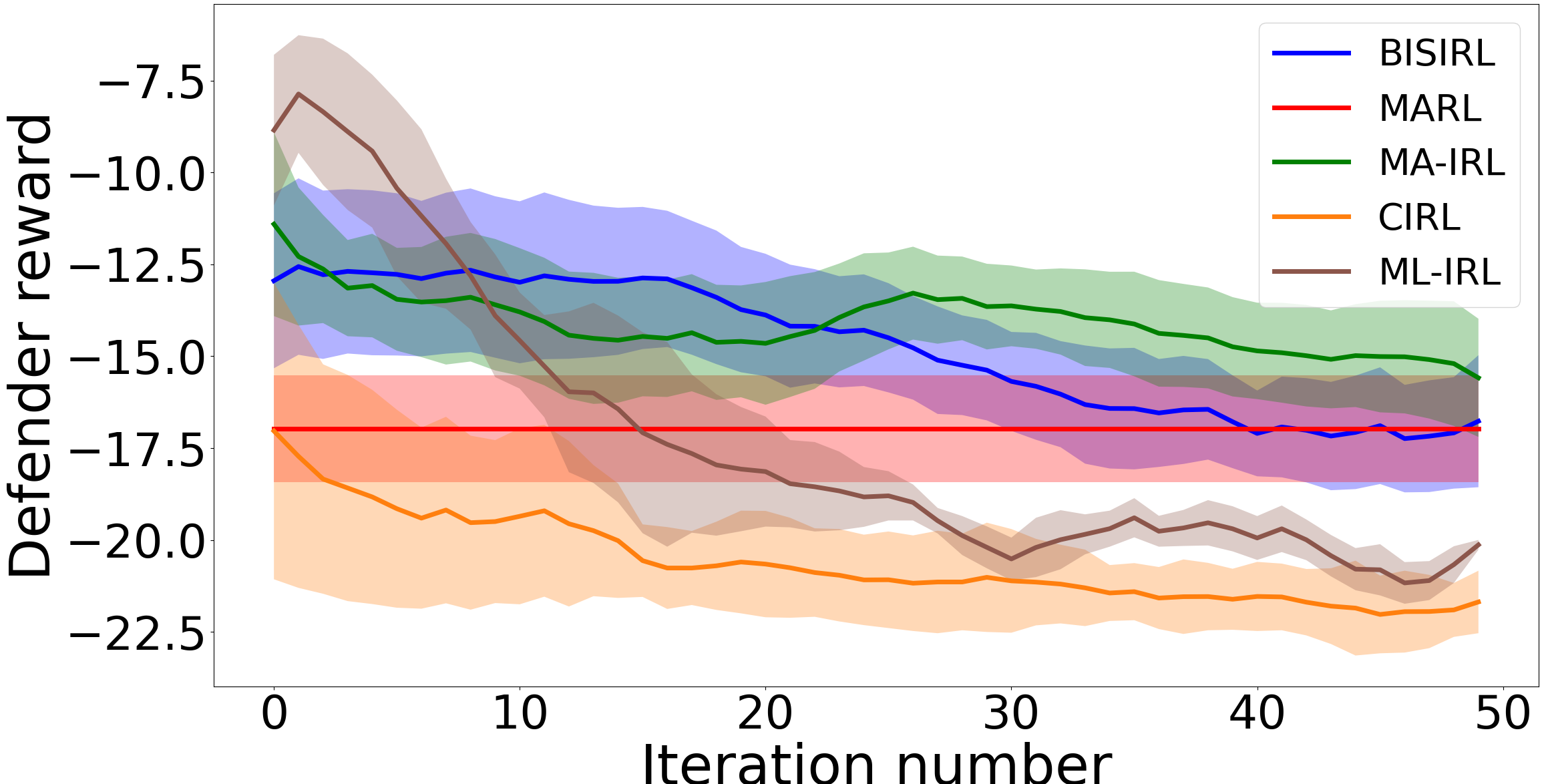}

}
\hspace{-1mm}
\subfigure{ 
\label{ar}
\includegraphics[width=4.25cm]{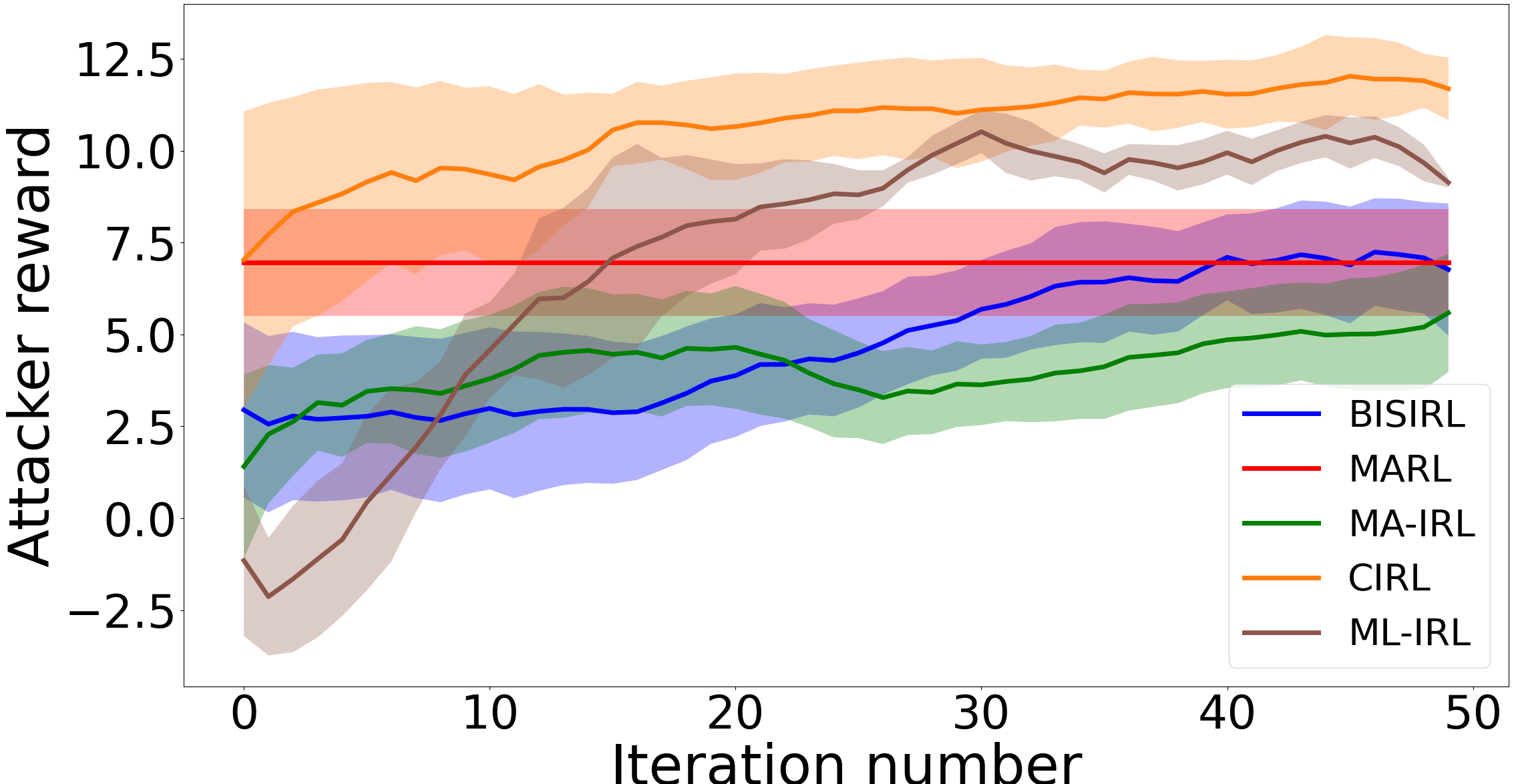}

}
\hspace{-1mm}
\vspace{-3mm}
\caption{Cyber security simulation results. \textbf{Left}: Attack graph. \textbf{Middle}: Defender's reward. \textbf{Right}: Attacker's reward. The cumulative rewards are calculated in the same way described in Figure \ref{mpeexp}.}
 \vspace{-3mm}
\label{csr}

\end{figure*}

We conduct a cybersecurity experiment to evaluate the proposed algorithm. The experimental setup involves a defender (learner) and an attacker (expert) interacting on the attack graph in Figure \ref{csr}. The attacker attempts to compromise nodes in the graph. The attacker's objective is to compromise as many nodes as possible, whereas the defender’s objective is to protect the network by minimizing the number of compromised nodes. Both the learner and the expert have discrete state and action spaces, with cardinalities of $256$ and $8$, respectively. 

Figure \ref{csr} shows that the cumulative rewards of the proposed algorithm converge to those of MARL, consistent with previous experiments.

There are $8$ nodes and $10$ edges.  Each node represents a machine, and each edge represents an exploit between two nodes. The decision-making of the defender and the attacker is modeled as an MG. The state $s \in \{0,1\}^8$, represents the condition of each node where the value $1$ means the current node is compromised by the attacker, and the value $0$ is vice versa. In each action pair, the attacker chooses one edge to attack, and the defender chooses one edge to block. Suppose the attacker chooses to attack the edge $\{i,j\}$. If the node $i$ is already compromised and the defender does not block this edge, there is a probability for node $j$ to be compromised. For other situations, the node $j$ keeps clean. Each edge has a cost for the attacker to utilize and a cost for the defender to block. The attacker receives a reward when it successfully compromises a new node. The net reward of the attacker for each state-action pair is the sum of the reward and the cost. For the defender, the reward is the sum of the opposite of the attacker's reward and th the cost to block edges.

For the security simulation, the attack graph is randomly generated. We use Q-learning to find the policies for the attacker and the defender. During the training process, the attacker and the defender have a $70\%$ possibility to choose between the best action with $60\%$ possibility and the second best action with $40\%$ possibility. Otherwise, the attacker and the defender randomly choose one action from the action space. When we exploit the learned policies, the attacker and the defender choose between the best action with $60\%$ possibility and the second best action with $40\%$ possibility.

\section{Limitation}\label{limitation}
Currently, the proposed framework is based on the fully observed MG. However, in some cases, the learner can not observe all the information about the environment. Therefore, the limitation of the proposed framework is cannot directly apply to the partially observed situation, and we will develop an advanced algorithm to solve this limitation in the future.


\end{document}